\documentclass[reqno,12pt]{amsart}
 
\usepackage{amsfonts,amsmath,amssymb,amsthm,mathtools}
\usepackage[foot]{amsaddr}
\usepackage{url}
\usepackage[colorlinks=true,citecolor=blue]{hyperref}
\usepackage{natbib}
\setcitestyle{round}
\usepackage[toc,page]{appendix}
\usepackage{latexsym}
\usepackage{enumerate}
\usepackage{import}
\usepackage{xcolor}
\usepackage[utf8]{inputenc}
\usepackage{algorithm}
\usepackage{comment}
\usepackage[noend]{algpseudocode}
\usepackage{multirow} 
\usepackage{graphicx}
\usepackage{subfig}
\usepackage{circledsteps}
\usepackage{enumerate}
\usepackage{verbatim}
\usepackage{pgfplots,pgfplotstable}
\pgfplotsset{compat=1.18}
\usepackage{bbm}
\usepackage{tikz}
\usetikzlibrary{arrows,positioning,chains,fit,shapes,calc,decorations}
\usepackage{lipsum}
\DeclareMathOperator{\eps}{\varepsilon}
\DeclareMathOperator{\proba}{\mathbb{P}}
\DeclareMathOperator{\AAA}{\mathcal{A}}
\DeclareMathOperator{\XX}{\mathcal{X}}

\DeclareMathOperator{\ZZ}{\mathcal{Z}}

\DeclareMathOperator{\CC}{\mathcal{C}}
\DeclareMathOperator{\BB}{\mathcal{B}}

\DeclareMathOperator{\NN}{\mathcal{N}}

\DeclareMathOperator{\PP}{\mathcal{P}}
\DeclareMathOperator{\RRR}{\mathcal{R}}
\DeclareMathOperator{\NNo}{{\mathbb{N}}_{0}}
\DeclareMathOperator*{\Esup}{ess\,sup}
\DeclareMathOperator{\aaa}{\mathbf{a}}
\DeclareMathOperator{\uu}{\mathbf{u}}

\newcommand{\norm}[1]{\left\|#1\right\|}

\newcommand{\normtv}[1]{\left\|#1\right\|_{\text{TV}}}
\newcommand{\prob}[1]{\mathbb{P}\left(#1\right)}

\newcommand{\RR}[1]{\mathbb{R}^{#1}}
\newcommand{\expo}[1]{\text{exp}\left(#1\right)}
\newcommand{\lnn}[1]{\ln\left(#1\right)}
\newcommand{\expected}[2]{\mathbb{E}_{#1}\left[#2\right]}

\newtheorem{theorem}{Theorem}[section]
\newtheorem{lemma}[theorem]{Lemma}
\newtheorem{corollary}[theorem]{Corollary}
\newtheorem{proposition}[theorem]{Proposition}

\newtheorem{remark}[theorem]{Remark}

\newtheorem{definition}[theorem]{Definition}

\title[Langevin under a Modulus of Continuity]{Mixing Times and Privacy Analysis for the Projected Langevin Algorithm under a Modulus of Continuity}

\author[M. Bravo]{Mario Bravo$^{\dagger}$}
\thanks{}
\address{$^{\dagger}$Departamento de Administraci\'on, Facultad de Administraci\'on y Econom\'ia, Universidad de Santiago de Chile, Av. Libertador Bernardo O’Higgins 3363, Santiago, Chile} 
\email{ \href{mailto:mario.bravo.g@usach.cl}{\nolinkurl{mario.bravo.g@usach.cl}}}

\author[J.P. Flores-Mella]{Juan Pablo Flores-Mella$^{\ast}$}
\address{$^{\ast}$Facultad de Matemáticas, Pontificia Universidad Cat\'olica de Chile, Vicu\~na Mackenna 4860, Santiago, Chile }
\email{ \href{mailto:jpflores1@uc.cl}{\nolinkurl{jpflores1@uc.cl}}}

\author[C. Guzm\'an]{Crist\'obal Guzm\'an$^{\ast,\ddagger}$}
\address{$^{\ddagger}$Institute for Mathematical and Computational Engineering, Facultad de Matem\'aticas and School of Engineering, Pontificia Universidad Cat\'olica de Chile, Vicu\~na Mackenna 4860, Santiago, Chile}
\email{ \href{mailto:crguzmanp@uc.cl}{\nolinkurl{crguzmanp@uc.cl}}}
\begin{document}

\begin{abstract}
We study the mixing time of the projected Langevin algorithm (LA) and the privacy curve of noisy Stochastic Gradient Descent (SGD), beyond nonexpansive iterations. Specifically, we derive new mixing time bounds for the projected LA which are, in some important cases, dimension-free and poly-logarithmic on the accuracy, closely matching the existing results in the smooth convex case. Additionally, we establish new upper bounds for the privacy curve of the subsampled noisy SGD algorithm. These bounds show a crucial dependency on the regularity of gradients, and are useful for a wide range of convex losses beyond the smooth case. Our analysis relies on a suitable extension of the Privacy Amplification by Iteration (PABI) framework \citep{feldman2018privacy,altschuler2022privacy,altschuler2022resolving} to noisy iterations whose gradient map is not necessarily nonexpansive. This extension is achieved by designing an optimization problem which accounts for the best possible R\'enyi divergence bound obtained by an application of PABI, where the tractability of the problem is crucially related to the modulus of continuity of the associated gradient mapping. We show that, in several interesting cases --namely the nonsmooth convex, weakly smooth and (strongly) dissipative-- such optimization problem can be solved exactly and explicitly, yielding the tightest possible PABI-based bounds.

\end{abstract}

\maketitle

\begin{small}
  \noindent{\bf Keywords:}  Privacy Amplification by Iteration, Langevin Algorithm,  Noisy Stochastic Gradient Descent, Differential Privacy
  \end{small}

%%%%%%%%%%%%%%%%%%%%%%%%%%%%%%%%%%%%%%%%%%%%%%%%%%%%%%%%%%

%%%%%%%%%%% INTRO %%%%%%%%%%%%%%%%%%%%%%%

\section{Introduction}

Sampling from a log-concave distribution $\pi$ (i.e., $\pi\propto e^{-f}$, where $f$ is a convex potential) is a fundamental algorithmic problem and a basic building block for problems such as volume estimation \citep{kannan1997random}, optimization \citep{Kalai:2006}, Bayesian statistics \citep{welling2011bayesian}, machine learning \citep{Ho:2020}, and differential privacy \citep{mcsherry2007mechanism}. There is a wide variety of algorithms designed to solve this problem, with the Langevin algorithm \eqref{LA} as one of the prominent examples. The idea is to consider the Euler-Maruyama discretization of the Langevin diffusion
\begin{equation*}
    dL_t=-\nabla f(L_t)dt+\sqrt{2}dW_t\quad (t\geq 0),
\end{equation*}
where $W_t$ is the $d$-dimensional Brownian motion. It is well-known that under mild assumptions, the diffusion has $\pi\propto e^{-f}$ as its unique stationary distribution --referred to as the \textit{target distribution}. The rationale is that by discretizing the diffusion with a small step $\eta>0$, we can use the Markov chain
\begin{equation}\label{LA}\tag{LA}
    X_{t+1}=X_t-\eta\nabla f(X_t)+\sqrt{2\eta}\xi_{t}\quad (t\in\NNo),
\end{equation}
where $(\xi_t)_t$ are i.i.d.~standard $d$-dimensional Gaussians, to (approximately) simulate $\pi$.

Recently,
\citet{altschuler2022privacy,altschuler2022resolving} have made major progress on understanding the procedure defined by 
\begin{equation}\label{eqn:langevin_dynamic} \tag{PLA}
 X_{t+1}=\Pi_{\XX}[X_t-\eta\nabla f(X_t)+\sigma\xi_{t}]\quad (t\in\NNo),
\end{equation}
where $(\xi_t)_t$ are i.i.d. $d$-dimensional Gaussians, $\XX \subseteq \RR{d}$ is a compact and convex set and $\Pi$ is the projection operator. Their analysis focuses particularly on two cases: $(i)$ when $\sigma=\sqrt{2\eta}$, referred to as the \textit{Projected Langevin Algorithm}, first introduced in \cite{BubeckEldanLehec}, for which they establish mixing times, and $(ii)$ when $\sigma=O(\eta)$, corresponding to \textit{Noisy Stochastic Gradient Descent}, for which they investigate the privacy curve.

The thrust of their analysis is based on a technique known as {\em Privacy Amplification by Iteration} (henceforth, PABI) \citep{feldman2018privacy}, which leverages the nonexpansive properties of the gradient step in the smooth convex setting to gradually and recursively control the R\'enyi divergence of iterates, either under different initializations (used for mixing time arguments) or potentials (used for privacy arguments). Given the power of this technique and its potential to aid in understanding both privacy and sampling, we consider it important to study the PABI technique beyond the smooth convex scenario. We highlight that among the cases that we study are the convex and $L$-Lipschitz, which encompasses functions that can be nondifferentiable, the convex and $(p,M)$-weakly smooth, which interpolates between the Lipschitz and the smooth one, and the smooth strongly dissipative case. See Table~\ref{table} for a more precise summary.

\subsection{Our Results}
    In this work, we conduct a study of the PABI technique beyond the case of nonexpansive iterations, together with some consequences for the mixing time and privacy analysis of this algorithm.

 \textit {Extension of PABI for general mappings in terms of the modulus of continuity.} 
    We start by providing an extension of the PABI technique to iterations beyond the nonexpansive case. In order to do this, we quantify the regularity of the underlying mapping 
    by its {\em modulus of continuity}. More precisely, for a vector valued map $\Phi:\RR{d}\to \RR{d}$, the nondecreasing function $\varphi:\RR{}_{+}\to\RR{}_{+}$ is a modulus of continuity of $\Phi$ if
        $\norm{\Phi(x)-\Phi(y)}\leq \varphi(\norm{x-y})$ for all $x,y$. An interesting feature of this extension is that we can even address discontinuous mappings (characterized  by their moduli of continuity being discontinuous at the origin). For instance, this extension is crucial for studying PABI under convex Lipschitz potentials, which are only subdifferentiable.

    PABI works by gradually interpolating between a worst-case distance bound (quantified by the $\infty$-Wasserstein distance) and a R\'enyi divergence bound. This interpolation is performed by using the {\em shifted R\'enyi divergence}, which is an infimal convolution between the convex indicator of a $\infty$-Wasserstein ball (with radius given by the shift) and the R\'enyi divergence. 
    By using a shift-reduction property of Gaussian noise addition, in conjunction with the nonexpansiveness of the gradient mapping, one can gradually reduce the shifts on the shifted divergences at the expense of an increase in the upper bound. This process concludes with zero shift, i.e.~an upper bound on the R\'enyi divergence. Notice in particular that the shifts applied at the different steps are tunable parameters, which in the case of nonexpansive mappings are easily optimized by uniform shifts. In our case, the modulus of continuity leads to a nonconvex optimization problem in terms of the tuning parameters. Remarkably, when the modulus of continuity of the iteration is of the form $\varphi(\delta)=\sqrt{c\delta^2+h}$, where $c,h\geq0$ are two parameters\footnote{It is understood that for any function the modulus of continuity is such that $\varphi(0)=0$. Hence, we will use formulae as above to refer to $\varphi$ for strictly positive values. With this in mind, it is clear that $\varphi(\delta)=\sqrt{c\delta^2+h}$ is discontinuous at zero if and only if $h>0$.}, this optimization problem has a unique optimal solution with a closed-form expression. 
    
    The list below highlights important contexts where this type of modulus of continuity arises. Table \ref{table} provides expressions for these moduli and also the bounds on the Rényi divergence of the final iterate, obtained from the application of Theorem~\ref{cor:minimized_E} to each case. To the best of our knowledge, all bounds in Table~\ref{table} are new. We regard this as our main technical contribution.
     
     Let $f:\XX\subseteq\RR{d}\to\RR{}$ be a function where $\XX$ is a closed convex set. We say that
     
\vspace{1.5ex} 
\noindent $\bullet$   $f$ is convex if for all $0\leq \lambda\leq 1$, and $x,y\in \XX$, 
        $f\left(\lambda x+(1-\lambda)y\right)\leq \lambda f(x)+(1-\lambda )f(y).$

\vspace{1.5ex} 

\noindent $\bullet$ $f$ is $(\lambda,\kappa)$-strongly dissipative\footnote{The denomination of strongly dissipative is not standard in the literature. We introduce it to distinguish it from the more standard notion of dissipativity, where $y$ is a fixed vector.} if there exist $\lambda,\kappa>0$ such that for all $x,y\in\XX$, $\langle \nabla f(x)-\nabla f(y),x-y\rangle \geq -\lambda+\kappa\left\|x-y\right\|^2$.

  \vspace{1.5ex} 
  
\noindent $\bullet$ $f$ is $L$-Lipschitz if there exists $L>0$ such that for all $x,y\in \XX$,
                $|f(x)-f(y)|\leq L\norm{x-y}.$

  \vspace{1.5ex} 
  
 \noindent $\bullet$ $f$ is $(p,M)$-weakly smooth (or have $p$-H\"older continuous gradient) if there exist $M>0$ and $0\leq p\leq 1$ such that for all $x,y\in \XX$,
                $\norm{\nabla f(x)-\nabla f(y)}\leq M\norm{x-y}^{p}.$

  \vspace{1.5ex} 

\noindent $\bullet$ $f$ is $\beta$-smooth if there exist $\beta>0$ such that for all $x,y\in \XX$,
                $\norm{\nabla f(x)-\nabla f(y)}\leq \beta\norm{x-y}.$
  \vspace{1.5ex} 
  
See Section \ref{section:Preliminaries} for further details.

 \vspace{1.5ex} 
\begin{table}[ht]
    \newcolumntype{M}[1]{>{\centering\arraybackslash}m{#1}}
    \newcolumntype{N}{@{}m{0pt}@{}}
    \begin{tabular}{|M{2cm}|M{5.4cm}|M{2.1cm}|M{3.6cm}|N}
    \hline
      {\small Assumptions on $f$} & {\small R\'enyi divergence of order $\alpha$} & $c$ & $h$\\
      \hline
      {\small Convex, $L$-Lipschitz} & $\frac{\alpha}{2\sigma^2}\left(\frac{D^2}{ T}+h\sum_{t=1}^{T}\frac{1}{t}\right)$ & $1$ & $(2\eta L)^2$\\
      \hline
       {\small Convex, $(p,M)$-w.s.} & $\frac{\alpha}{2\sigma^2}\left(\frac{D^2}{ T}+h\sum_{t=1}^{T}\frac{1}{t}\right)$ & $1$ & $\Big(\!2\eta^{\frac{1}{1\!-\!p}}\!\mbox{$\sqrt{\frac{1-p}{1+p}}$}\left(\frac{M}{2}\right)^{\frac{1}{1\!-\!p}}\!\Big)^2$\\
      \hline
      {\small $(\lambda,\kappa)$-Str. Dissip, $\beta$-smooth}  & $ \frac{\alpha}{2\sigma^2}\!\!\left(\!\frac{D^2 c^T(1-c)}{(1-c^T)}\!+\!h\ln\left(\left(\frac{1-c^T}{1-c}\right)\! e\right)\right)$ & $1\!-\!2\eta \kappa \!+\! \eta^2\beta^2$ & $2\eta \lambda$\\
      \hline
    \end{tabular}
    \caption{Summary of the R\'enyi divergence bounds between the last iterates under two different initializations of \eqref{eqn:langevin_dynamic}. For all rows, the corresponding moduli of continuity can be bounded by $\varphi(\delta)=\sqrt{c\delta^2+h}$, with $c,h$ given in the corresponding columns. Here $D$ is the diameter of $\XX$, $T$ is the number of iterations in the algorithm and $\sigma^2$ is the (coordinate-wise) variance for Gaussian noise. For more information, see the list below Definition \ref{def:modulus_of_continuity}.}\label{table}
\end{table}

 \textit {Mixing times.}
    We show that the PABI technique yields a polynomial upper bound on the mixing time in total variation distance of the projected Langevin algorithm in the convex and nonsmooth case, including cases where the potential is only subdifferentiable. We also provide a mixing time bound for the strongly dissipative case, which is logarithmic in the diameter, but exponential in the parameter $\lambda$. The following are informal versions of the Theorems, whose complete statements can be found in Section \ref{section:mixing_time}.
    \begin{theorem}[Abridged version of Theorem~\ref{thm:mixing_for_weakly_smooth}]
    Let $\XX\subseteq\RR{d}$ be a convex, compact set with diameter $D>0$ and suppose that $f:\XX\to\RR{}$ is a convex and $(p,M)$-weakly smooth function, with $0\leq p\leq 1$. There exists a constant $\Theta$ such that if $1/\eta\geq \Theta$,     then for all $\eps>0$,
        $$T_{mix,TV}(\eps)\leq \left\lceil\frac{D^2}{\eta}\right\rceil\cdot\lceil\log_2(1/\eps)\rceil.$$
        
\end{theorem}
  We remark that $\Theta$ depends polylogarithmicaly in the diameter $D$ and polynomially in $M,p$; furthermore, $p$ modulates these dependencies, making them near-quadratic on $M$ when $p=0$ (Lipschitz case) and linear on $M$ when $p=1$ (smooth case). For an explicit expression, see \eqref{eqn:bd_eta}. Also, we stress that our upper bound operate under a slightly more restrictive stepsize constraint, but otherwise the bound on mixing time is identical to that of the smooth one \citep{altschuler2022resolving}, for all $0\leq p\leq 1$. This includes the whole interpolation from the convex and nonsmooth ($p=0$) to the convex and smooth case ($p=1$).
\begin{theorem}[Abridged version of Theorem~\ref{thm:mixing2}]
    Let $\XX\subseteq\RR{d}$ be a convex, compact set with diameter $D>0$ and suppose that $f:\XX\to\RR{}$ is a $(\lambda,\kappa)$-strongly dissipative and $\beta$-smooth function. If $c=1-2\eta\kappa+\eta^2\beta^2<1$, then for all $\eps>0$,
    \begin{equation*}
        T_{mix,TV}(\eps)= O\left(\log_{1/c}\left(1+\frac{D^2(1-c)}{4\eta}\right)\cdot \left(\frac{e}{1-c}\right)^{\lambda/2}\log_2(1/\eps)\right).
    \end{equation*}
\end{theorem}

    \textit {Privacy curve.}
    We also study the impact of our PABI results for the privacy curve of noisy SGD, in the convex setting. For any nontrivial  H\"older gradient regularity, we have that the privacy curve caps in a similar fashion to that proven in \citep{altschuler2022privacy}, except for the addition of an extra term (denoted by $V$ in the corresponding section) that depends on $\eta$ (the step) and the Hölder regularity of the gradient. Particularly, we prove that for $L$-Lipschitz and $(p,M)$-weakly smooth losses, and under mild restrictions over the R\'enyi divergence parameter $\alpha\geq 1$, variance $\sigma^2$, and number of iterations $T$; noisy SGD satisfies the following R\'enyi DP bound.

    \begin{theorem}[Abridged version of Theorem~\ref{thm:privacy_analysis_nsgd}]
        Let $\XX\subseteq\RR{d}$ be a convex and compact set with diameter $D>0$. There exists $\overline{T}>0$ and a function $V$ such that for $T>\overline{T}$, datasize $n\in\mathbb{N}$, expected batch size $b$, stepsize $\eta>0$ and initialization $x_0\in\XX$, the last iteration satisfies $(\alpha,\eps)$-RDP for
         \begin{equation*}
            \eps\leq\frac{16\alpha L^2}{n^2 \sigma^2} \min\Big\{T,2\overline{T}+V(D,M,\overline{T},\eta,p)\Big\}.
        \end{equation*}
    \end{theorem}
    
    This shows that for convex and $(p,M)$-weakly smooth losses, the privacy curve caps in a similar fashion to that of the smooth and convex ones, except for an additive term $V$. For details, see Section \ref{section:privacy_analysis}. In particular, see Figure \ref{figure:privacy_curve_caps} for a plot comparing the privacy curves.
    However, our conclusions for the nondifferentiable case fall short: it is not possible to obtain any nontrivial privacy amplification, even when the sample size tends to infinity. 
    Note however that due to the optimality of our PABI optimization problem, and the tighness of the
    modulus of continuity for the gradient mapping we use, 
    these pessimistic results exhibit 
    the inherent limits of PABI in the nonsmooth convex setting.

\subsection{Related Work}

        A substantial part of the studies for the Langevin algorithm have focused on the strongly convex and smooth potential setting   (e.g.~\citealt{Dalalyan:2017,dalalyan2017further,durmus2018highdimensional}). 
        Most of the research here has focused on approximation bounds (e.g.~in Wasserstein or total variation distance) between the last iterate of LA with the target distribution, $\pi$. These arguments are based on using the stationarity of the target distribution under the difussion, together with a coupling between the discrete and continuous Langevin dynamics, to control the distance. These results lead to bounds that blow up with respect to the time (of the discrete chain), and therefore these are inherently finite-iteration statements. This body of work primarily targets approximation to the stationary distribution, $\pi$, of the diffusion, typically in unconstrained Euclidean settings. By contrast, the PABI approach undertaken in \citet{altschuler2022resolving}, provides a mixing time bound of the Projected Langevin algorithm to its own stationary distribution. Although this distribution may not be the target distribution $\pi$, for some purposes this approximation in unnecessary: this is the case e.g.~in differentially-private optimization, where the goal is bounding the empirical or population risk, together with a divergence bound (e.g.~R\'enyi) between two outputs using datasets that only differ in a single example. For references on the capabilities of noisy iterative methods for this problem, we refer the reader to  \citep{bassily2014private, bassily2020stability}.

        Far less work has been devoted to the case of nonsmooth convex potentials; in fact, several works consider the Langevin algorithm for nonsmooth potentials as far less understood (e.g.~\citealt{Pereyra:2016,Chatterji:2020,Mitra:2024}). A first natural approach to reduce the nonsmooth setting to a smooth one is using a convolution-type smoothing. This includes the case of proximal algorithms (e.g.~\citealt{Pereyra:2016,Durmus:2018,Wibisono:2019}), or randomized smoothing (e.g.~\citealt{Chatterji:2020}). Regarding the former, while it can be preferable to use proximal smoothing due to its stability properties, these methods require  proximal mapping computations, which are tractable only for very structured cases. Regarding the latter, even if randomized smoothing is easily implementable, the smoothness of the resulting functions have polynomial dependence on the dimension, which leads to sample complexity results which are quite expensive. All these works focus on the approximation to the diffusion’s stationary distribution, without any projection.

        To the best of our knowledge, the only existing work that analyzes the (Projected) Langevin algorithm in the convex Lipschitz case is \citet{Lehec:2023}. This work extends the coupling techniques used in the smooth case \citep{Dalalyan:2017}, observing that the monotonicity of the gradient suffices for nonexpansiveness of the Langevin difussion, and using discrete/continuous time couplings establishes approximation results in the Wasserstein metric. \cite{Lehec:2023}  covers both constrained (compactly supported) and unconstrained domains, illustrating how the same coupling techniques can be adapted to different geometric settings.
        
    We emphasize that almost all of these results are independent and incomparable to ours, as they do not imply mixing-time bounds. On the other hand, our results do not focus on the approximation to the target distribution. Additionally, upon completing this paper, we learned of related work by \cite{johnston2025performanceunadjustedlangevinalgorithm}, which establishes approximation results to~$\pi$ in the Wasserstein-1 and Wasserstein-2 metrics for potentials that are semiconvex on a ball and strongly convex outside it. Their analysis also accommodates discontinuous gradients by employing subgradient techniques. The results are of a different nature from ours, due to the differing assumptions and approximation criteria.

In the case of differential privacy, classical analyses of differentially private iterative algorithms assume that all iterates are published, leading to unbounded growth in privacy parameters with the number of iterations. It is often the case however that only the last iteration is published. Recent works have explored this setting. Among these, \citet{chourasia2021differential} and \citet{ye2022differentially} show that, in the smooth and strongly convex case, the R\'enyi Differential Privacy (RDP) of variants of noisy SGD approaches a constant bound exponentially quickly. \citet{altschuler2022privacy} show, using PABI, that in the smooth and (strongly) convex settings the RDP of projected noisy SGD stops growing after a certain number of iterations. \citet{asoodeh2023privacy} prove a convergent upper bound for the privacy of DP-SGD, even in nonconvex settings, using Hockey-Stick divergence.

We further note that several other recent studies have addressed related questions, highlighting the critical role of modulus of continuity in the privacy analysis of the last iterate of noisy SGD, underscoring the importance of this line of investigation. 
The first, \citet{kong2024privacyiteratecyclicallysampleddpsgd}, investigates the last iteration of noisy SGD without assuming random sampling of the data set, employing clipped gradients. Their analysis diverges from ours for several reasons: they operate under distinct hypotheses, leverage an affine modulus of continuity, and directly incorporate sensitivity into their methodology. The second paper, \citet{chien2024convergentprivacylossnoisysgd}, explores scenarios where the gradients exhibit H\"older continuity, allowing the objective function to be non-convex. This assumption leads to a modulus of continuity that differs from ours and yields distinct bounds. Moreover, they cannot analytically resolve their shift optimization problem. Their work also employs a variation of the PABI mechanism, which differs slightly from the diameter-aware version introduced by \citet{altschuler2022privacy}.

\subsection{Organization of the Paper} This paper is organized as follows. Section \ref{section:Preliminaries} presents the necessary background results for the reading of the paper. Section \ref{section:PABI} provides the extension of PABI to the modulus of continuity setting. Section \ref{section:mixing_time} presents mixing times in total variation for convex and Lipschitz, convex and $(p,M)$-weakly smooth and $(\lambda,\kappa)$-strongly dissipative and $\beta$-smooth settings. Section \ref{section:privacy_analysis} presents a privacy analysis for the last iteration of noisy SGD for convex and $(p,M)$-weakly smooth functions.

We also add several appendices to complement the presentation. Appendix \ref{appendix:definitions_and_results} presents a brief summary of useful results that may be used for quick inspection. Appendix \ref{app:convexity_E} shows that the problem we solve for the PABI extension is nonconvex, while Appendix \ref{appendix:existence_of_stationary_distributions} deals with the existence of a stationary distribution of \eqref{eqn:langevin_dynamic} with $\sigma^2=2\eta$, when the potential, $f$, is convex and $L$-Lipschitz  (possibly nondifferentiable).

%%%%%%%%%%% Preliminaries %%%%%%%%%%%%%%%%%%%%%%%

\section{Preliminaries}\label{section:Preliminaries}

\subsection{Vector Spaces and Convex Functions} We work over the standard Euclidean space $(\mathbb{R}^{d},\|\cdot\|)$ (that is $\|\cdot\|=\|\cdot\|_2$ is the $\ell_2$-norm). If $\XX$ is a closed convex set, we denote by $\Pi_{\XX}:\mathbb{R}^d\mapsto{\XX}$ the Euclidean projection operator, which we recall is nonexpansive. We denote by $I_{d\times d}$ the $d$-dimensional identity matrix.

We now introduce the modulus of continuity, which serves as a measure of regularity of functions, and it is the main property used in the PABI technique we will introduce later.

\begin{definition}[Modulus of continuity]\label{def:modulus_of_continuity}
    Let $\Phi:\XX\subseteq\RR{d}\to\RR{d}$ be a map. We say that a nondecreasing function $\varphi:\RR{}_{+}\to\RR{}_{+}$ is a modulus of continuity of $\Phi$ if
    \begin{equation*}
        \norm{\Phi(x)-\Phi(y)}\leq \varphi(\norm{x-y})\quad\forall x,y\in \XX.
    \end{equation*}
\end{definition}

    Note that $\lim_{t\to0^+}\varphi(t)=0$ implies $\Phi$ is continuous and that we can always assume that $\varphi(0)=0$.
    
    For a function $f:\XX\subseteq\RR{d}\to\RR{}$, given $\eta>0$, let $\Phi(x)=x-\eta\nabla f(x)$ be the gradient mapping. We have that 
    \begin{enumerate}
        \item If $f$ is convex and $L$-Lipschitz, then $\varphi(\delta)=\sqrt{\delta^2+(2\eta L)^2}$ \citep[Lemma~3.1]{bassily2020stability}.
        \item If $f$ is convex and $(p,M)$-weakly smooth, then $\varphi(\delta)=\sqrt{\delta^2+\left(2\eta^{\frac{1}{1-p}}\sqrt{\frac{1-p}{1+p}}\left(\frac{M}{2}\right)^{\frac{1}{1-p}}\right)^{2}}$ \citep[Lemma~D.3]{lei2020fine}.

        Note that we recover the modulus of continuity of the Lipschitz case when $p=0$. The only difference is in the Lipschitz constant that is now $(M/2)$. Thus, a $(0, 2L)$-weakly smooth function has the same modulus of continuity as a $L$-Lipschitz one.

        \item If $f$ is $(\!\lambda,\!\kappa\!)$-strongly dissipative and $\beta$-smooth, then $\varphi(\delta)\!\!=\!\!\sqrt{(1-2\eta\kappa+\eta^2\beta^2)\delta^2+2\eta\lambda}$.
    \end{enumerate}
    
            We were unable to find a reference providing a modulus of continuity bound in the strongly dissipative case; therefore, we include a brief explanation for completeness.
            Suppose $\left\|x-y\right\|\leq \delta$. Then
            \begin{align*}
                \left\|x-\eta\nabla f(x)-(y-\eta\nabla f(y))\right\|^2&=\left\|x-y\right\|^2-2\eta\langle \nabla f(x)-\nabla f(y),x-y\rangle + \eta^2\left\|\nabla f(x)-\nabla f(y)\right\|^2\\
                &\leq \delta^2 -2\eta \langle \nabla f(x)-\nabla f(y),x-y\rangle +\eta^2\beta^2\delta^2\\
                &\leq (1-2\eta\kappa+\eta^2\beta^2)\delta^2+2\eta\lambda,
            \end{align*}
            where in the first line we expand the square; in the second line we use the $\beta$-smoothness of $f$ and the bound on $\left\|x-y\right\|$; in the third line we use the $(\lambda,\kappa)$-strong dissipativity of $f$.

\subsection{Information Theory and Probability Divergences} In the following we use $\PP(\XX)$ to denote the set of all probability measures supported in $\XX$ and $\mathcal{B}(\XX)$ to denote the Borel $\sigma$-algebra of $\XX$.

\begin{definition}[R\'enyi Divergence]
    Let $\alpha\in (1,+\infty)$ and $\mu,\nu\in\PP(\RR{d})$ be two probability measures on $\RR{d}$. We define the R\'enyi divergence of order $\alpha$ between $\mu,\nu$ as:
    \begin{equation*}
        R_{\alpha}(\mu||\nu)=\begin{cases}
            \frac{1}{\alpha-1}\ln\left(\int_{\RR{d}}\left(\frac{d\mu}{d\nu}(x)\right)^{\alpha}\nu(dx)\right)&\text{ if }\mu\ll\nu\\
            +\infty&\text{ otherwise}.
        \end{cases}
    \end{equation*} 
\end{definition}

    It is well-known (e.g.~\citet[Theorem 5]{van2014renyi}) that if there exists $\beta>1$ such that $R_\beta(\mu||\nu)<\infty$, then $\mathrm{KL}(\mu||\nu)=\lim_{\alpha\to 1^{+}}R_\alpha(\mu||\nu)$.

    In a slight abuse of notation, we write R\'enyi divergences applied to random variables meaning the divergence of the respective distributions.

\begin{definition}[Coupling and $\infty$-Wasserstein Distance] Let $\mu,\nu\in \PP(\RR{d})$ be two probability measures. We say that $\gamma\in\PP(\RR{d})$ is a coupling of $\mu$ and $\nu$ if for all $A\in{\mathcal B}(\RR{d})$
\begin{equation*}
    \gamma(A\times \RR{d})=\mu(A)\qquad\text{and}\qquad \gamma(\RR{d}\times A)=\nu(A).
\end{equation*}
We denote by $\Gamma(\mu,\nu)$ the set of all couplings between $\mu$ and $\nu$.

We say that a pair of random variables is a coupling of $\mu$ and $\nu$ if its joint distribution is in $\Gamma(\mu,\nu)$; i.e.~ $X\sim \mu$ and $X^{\prime}\sim \nu$.

Finally, we define the $\infty$-Wasserstein distance between $\mu$ and $\nu$ as
\begin{equation*}
    W_{\infty}(\mu,\nu):=\inf_{\gamma\in\Gamma(\mu,\nu)}\Esup_{(x,y)\sim\gamma}\norm{x-y}.
\end{equation*}
    
\end{definition}

A key tool for analyzing is the so called \textit{Shifted R\'enyi Divergence}, which helps to interpolate between a $W_{\infty}$ guarantee --which holds trivially in the compact setting-- and a R\'enyi divergence one.

\begin{definition}[Shifted R\'enyi Divergence]
    Let $\alpha\in [1,+\infty)$, $\delta\geq 0$, and $\mu,\nu\in \PP(\RR{d})$ be two probability measures. We define the $\delta$-shifted R\'enyi divergence of order $\alpha$ as:
    \begin{equation*}
        R_{\alpha}^{(\delta)}(\mu||\nu)=\underset{\mu^{\prime}:W_\infty(\mu,\mu^{\prime})\leq \delta}{\inf} R_{\alpha}(\mu^{\prime}||\nu).
    \end{equation*}
\end{definition}

    Two useful properties of the shifted R\'enyi divergence are $R_{\alpha}^{(0)}(\mu||\nu)=R_{\alpha}(\mu||\nu)$, and that
        $W_{\infty}(\mu,\nu)\leq \delta$, implies $R_{\alpha}^{(\delta)}(\mu||\nu)=0$. These properties account for the final and initial bounds on PABI. A key property for PABI is the shift-reduction property of Gaussian noise addition \citep[Lemma~20]{feldman2018privacy}.

\begin{lemma}[Shift-reduction]\label{lemma:shift_reduction}
    For  $\mu,\nu\in\PP(\RR{d})$ and $a,\delta\geq 0$,
    \begin{equation*}
        R_\alpha^{(\delta)}\left(\mu\ast\NN\left(0,\sigma^2 I_{d\times d}\right)||\nu\ast \NN\left(0,\sigma^2 I_{d\times d}\right)\right)\leq R_{\alpha}^{(\delta+a)}\left(\mu||\nu\right)+\frac{\alpha a^2}{2\sigma^2}.
    \end{equation*}
\end{lemma}

%%%%%%%%%%% Shift optimization problem %%%%%%%%%%%%%%%%%%%%%%%

\section{Privacy Amplification by Iteration Under a Modulus of Continuity}\label{section:PABI}

We start by providing a simple extension of the PABI framework for nonexpansive iterations to the case of general maps under a modulus of continuity assumption. As discussed in the introduction, we want to study iterations of the form 
\[ X_{t+1}=\Pi_{\XX}[\Phi_t(X_t)+\xi_t], \qquad \xi_t\sim {\NN}(0,\sigma_t^2 I_{d\times d}). \]
In what follows, we denote by $\varphi_t$ the modulus of continuity of the map $\Phi_t$. We are interested in bounding the R\'enyi divergence of two different trajectories of this algorithm, either under different initializations (to obtain mixing time results) or under different maps $\Phi_t,\Phi_t^{\prime}$ (to prove privacy results). The following two results are an adaptation of  \citet[Lemma 21]{feldman2018privacy} and \citet[Theorem 22]{feldman2018privacy}, respectively.

\begin{lemma}[Coupling under modulus of continuity\label{lemma:nombre_pendiente}]
    Let $\mu, \nu\in\PP(\RR{d})$, $\delta\geq 0$ and\\ $\alpha\in [1,+\infty)$. Let also $\Phi:\RR{d}\to\RR{d}$ be a map with modulus of continuity $\varphi$. Then
    \begin{equation*}
        R_{\alpha}^{\left(\varphi(\delta)\right)}\left(\Phi_{\#}\mu||\Phi_{\#}\nu\right)\leq R_{\alpha}^{(\delta)}\left(\mu||\nu\right),
    \end{equation*}
    where $\Phi_{\#}\mu$ and $\Phi_{\#}\nu$ denote the pushforward measure of $\mu$ and $\nu$ through $\Phi$, respectively.
\end{lemma}

\begin{proof}
    Let $\mu^{\prime}$ be such that $W_{\infty}(\mu,\mu^{\prime})\leq \delta$ and
        $R_{\alpha}(\mu^{\prime}||\nu)= R_{\alpha}^{(\delta)}(\mu||\nu).$
    Let $(X,X^{\prime})$ be a coupling of $(\mu,\mu^{\prime})$ such that
        $\norm{X-X^{\prime}}\leq \delta$ a.s.\\ 
    Then $ \norm{\Phi(X)-\Phi(X^{\prime})}\leq \varphi(\norm{X-X^{\prime}})\leq \varphi(\delta)$ a.s.  
    Also, by the data-processing inequality (Proposition \ref{prop:data_processing_ineq}), 
    \begin{equation*}
        R_{\alpha}(\Phi_{\#}\mu^{\prime}||\Phi_{\#}\nu)\leq R_{\alpha}(\mu^{\prime}||\nu)=R_{\alpha}^{(\delta)}(\mu||\nu).
    \end{equation*}
    Therefore, since $(\Phi(X),\Phi(X^{\prime}))$ is a coupling of $(\Phi_{\#}\mu,\Phi_{\#}\mu^{\prime})$,
    \begin{equation*}
        R_{\alpha}^{(\varphi(\delta))}(\Phi_{\#}\mu||\Phi_{\#}\nu)\leq R_{\alpha}(\Phi_{\#}\mu^{\prime}||\Phi_{\#}\nu)\leq R_{\alpha}^{(\delta)}(\mu||\nu).
    \end{equation*}
\end{proof}

\begin{lemma}\label{lemma:mixture_of_the_lemmas}
    Let $\XX\subseteq\RR{d}$ be a closed convex set and $\Phi:\RR{d}\to\RR{d}$ be a map with modulus of continuity $\varphi$. Also, let $X,X^{\prime}$ be two random variables in $\RR{d}$ and $\xi\sim\NN(0,\sigma^2 I_{d\times d})$ be centered Gaussian noise with variance $\sigma^2$. Let $Y=\Pi_{\XX}\left[\Phi(X)+\xi\right]$ and $Y^{\prime}=\Pi_{\XX}\left[\Phi(X^{\prime})+\xi\right]$, 
    and let $\delta\geq 0$. Then, for any $0<a\leq \varphi(\delta)$,
    \begin{equation*}
        R_{\alpha}^{(\varphi(\delta)-a)}(Y||Y^{\prime})\leq R_{\alpha}^{(\delta)}(X||X^{\prime})+\frac{\alpha a^2}{2\sigma^2}.
    \end{equation*}
\end{lemma}

\begin{proof}
    By a succesive application of Lemmas \ref{lemma:nombre_pendiente}, \ref{lemma:shift_reduction} and \ref{lemma:nombre_pendiente} again, we have that:
    \begin{align*}
        R_{\alpha}^{(\delta)}(X||X^{\prime})
        &\geq
        R_{\alpha}^{(\varphi(\delta))}(\Phi(X)||\Phi(X^{\prime}))\\
        &\geq
        R_{\alpha}^{(\varphi(\delta)-a)}(\Phi(X)+\xi||\Phi(X^{\prime})+\xi)-\frac{\alpha a^2}{2\sigma^2}\\
        &\geq
        R_{\alpha}^{(\varphi(\delta)-a)}(\Pi_{\XX}\left[\Phi(X)+\xi\right]||\Pi_{\XX}\left[\Phi(X^{\prime})+\xi\right])-\frac{\alpha a^2}{2\sigma^2},
    \end{align*}
where in the last step we used the nonexpansiveness of $\Pi_{\XX}$.
\end{proof}

We introduce now a simplifying notation for the PABI induction. Given 
$T\in\mathbb{N}$, $D>0$ and $\aaa=(a_t)_{t=1}^{T}$ be a sequence of nonnegative reals. We define $ \varphi_{[0:0]}(D,\aaa):=\varphi_0(D)-a_1$, and
\[\varphi_{[0:t]}(D,\aaa):=\varphi_t\left(\varphi_{[0:t-1]}(D,\aaa)\right)-a_{t+1}, \quad  \text{for} \quad t=1,\ldots,T-1.\]

\begin{lemma}[Privacy amplification by iteration under a modulus of continuity] \label{lemma:privacy_amp}
    \text{ }
    Let $\XX\subseteq\RR{d}$ be a convex and compact set with diameter $D>0$. Let $\mu_0, \mu_0^{\prime}\in{\PP}(\XX)$ be two probability measures. For every $t\in\NNo$, let $\Phi_t:\XX\subseteq\RR{d}\to\RR{d}$ be a mapping with modulus of continuity $\varphi_t$ and let $(\xi_t)_{t\in\NNo}\overset{i.i.d.}{\sim}\NN(0,\sigma_t^2 I_{d\times d})$. Define $(X_t)_{t\in\mathbb{N}_0}$ and $(X_t^{\prime})_{t\in\mathbb{N}_0}$ respectively as
    
\noindent\begin{minipage}{.5\linewidth}
    \begin{align*}
        X_0 &\sim\mu_0\\
        X_{t+1} &= \Pi_{\XX}\left[\Phi_{t}(X_t)+\xi_{t}\right],
    \end{align*}
    \end{minipage}
   \hspace{-7ex} and \hspace{-5ex}
    \begin{minipage}{.5\linewidth}
    \begin{align*}
        X_0^{\prime} &\sim\mu_0^{\prime}\\
        X_{t+1}^{\prime} &= \Pi_{\XX}\left[\Phi_{t}(X_t^{\prime})+\xi_{t}\right].
    \end{align*}
\end{minipage}
\vspace{2ex}
   
    Let $\aaa=(a_t)_{t=1}^{T}$ be such that $a_t\geq 0$ and $\varphi_{[0:t-1]}(D,\aaa)\geq 0$ for all $t\leq T$. Then
    \begin{equation}\label{eqn:bound_on_shifted_renyi_divergence_of_the_algorithm}
        R_{\alpha}^{\left(\varphi_{[0:T-1]}(D,\aaa)\right)}(X_T||X_T^{\prime})\leq \frac{\alpha}{2}\sum_{t=1}^{T}\frac{a_t^2}{\sigma_{t-1}^2}.
    \end{equation}
\end{lemma}

\begin{proof}
    The proof follows by iteratively applying Lemma \ref{lemma:mixture_of_the_lemmas} starting from $R_{\alpha}^{(D)}(X_0||X_0^{\prime})=0$.
\end{proof}

We will refer to those sequences $(X_t)_{t\in\NNo}$ having form as in Lemma \ref{lemma:privacy_amp} as \textit{Projected Noisy Iterations with moduli of continuity $(\varphi_t)_t$}.

\subsection{The Shifts Optimization Problem}\label{section:shifts_optimization_problem}

    We focus now on the optimization of parameters to obtain the tightest possible PABI bound. In order to bound $R_{\alpha}(X_T||X_T^{\prime})$, we need to find a suitable sequence $\aaa=(a_t)_{t=1}^{T}$ such that 
            $\varphi_{[0:T-1]}(D,\aaa)=0.$
    This is because, from Lemma~\ref{lemma:privacy_amp} and the fact that the $0$-shifted R\'enyi Divergence is the R\'enyi Divergence, we can obtain a bound for $R_{\alpha}(X_T||X_T^{\prime})$. 
        Since there are various \textit{feasible shifts} $\aaa$ to choose from, our goal is to  minimize the right hand side of equation \eqref{eqn:bound_on_shifted_renyi_divergence_of_the_algorithm}. More precisely, for $D>0$,  a sequence $\mathbf{a}=(a_t)_{t=1}^{T}$ of nonnegative reals is a sequence of feasible shifts if for all $t=1,\ldots,T$
    \begin{equation*}
        a_t \geq 0, \quad \varphi_{[0:t-1]}(D,\aaa)\geq 0, \quad \text{and} \quad \varphi_{[0:T-1]}(D,\aaa)=0.
    \end{equation*}

Note that $\aaa=(a_t)_{t=1}^{T}$ is a feasible shift if and only if
        $a_t=\varphi_{t-1}(u_{t-1})-u_t$ for all $t=1,\ldots,T$,
    where $(u_{t})_{t=0}^{T}$ is a sequence of nonnegative numbers that satisfies
\begin{equation}\label{feasible_shifts}
u_0=D, \quad u_T=0, \text{ and } \quad \varphi_{t-1}(u_{t-1})\geq u_t\quad\forall t=1,\ldots,T.
\end{equation}

Hence we can restate the problem of finding a sequence of feasible shifts $(a_t)_{t=1}^{T}$ that minimizes the right-hand-side of \eqref{eqn:bound_on_shifted_renyi_divergence_of_the_algorithm}
by the equivalent problem of finding a sequence of nonnegative real numbers $(u_t)_{t=0}^{T}$ that satisfies \eqref{feasible_shifts} and minimizes

\begin{equation*}
    \frac{\alpha}{2}\sum_{t=1}^{T}\frac{\left(\varphi_{t-1}(u_{t-1})-u_t)\right)^2}{\sigma_{t-1}^2}.
\end{equation*}

Let us call $\RRR$ the set of parameters $\uu=(u_1,\ldots,u_{T-1})\in\RR{T-1}$ that satisfy \eqref{feasible_shifts}.

Then, in order to obtain the tightest possible PABI upper bound, we consider the problem
\begin{equation}\tag{P}\label{problem:minimization_phi}
    \min_{\uu\in\RRR}\left[ E (\uu):=\sum_{t=1}^{T}\frac{(\varphi_{t-1}(u_{t-1})-u_{t})^2}{\sigma_{t-1}^2}\right].
\end{equation}

\subsection{Solving the Shifts Optimization Problem}

We now study the shifts optimization problem under a modulus of continuity assumption that encompasses families of both nonsmooth and smooth potentials.

While under the studied modulus of continuity the objective $E$ has a positive-definite Hessian at every point of $\RRR$, in general ${\RRR}$ is a nonconvex domain, which prevents us from a simple first-order condition characterization of optimality (see Appendix \ref{app:convexity_E} for more details). We will nevertheless characterize the first-order conditions, and then show by alternative arguments that this is indeed an optimal solution for problem \eqref{problem:minimization_phi}. 

We present next what we deem as our main result.

\begin{theorem}\label{cor:minimized_E}
Let $(c_t)_{t\in\NNo}$ be a sequence of strictly positive real numbers, $(h_t)_{t\in\NNo}$ a sequence of nonnegative real numbers, and $\varphi_t(\delta)=\sqrt{c_t\delta^2+h_t}$.

If $(X_t)_{t\in\NNo}$ and $(X_{t}^{\prime})_{t\in\NNo}$ are projected noisy iterations with moduli of continuity $(\varphi_t)_t$, which only differ in their initialization and whose domain, $\XX$, has diameter $D>0$, then
\begin{equation*}
    R_{\alpha}(X_T||X_T^{\prime})\leq
    \frac{\alpha}{2}\left(\frac{\Pi_{k=0}^{T-1}c_k D^2}{\sum_{j=0}^{T-1}\sigma_j^2\Pi_{l=j+1}^{T-1}c_l}+\sum_{t=0}^{T-1}\frac{h_t\Pi_{k=t+1}^{T-1}c_k}{\sum_{j=t}^{T-1}\sigma_j^2\Pi_{l=j+1}^{T-1}c_l}\right).
\end{equation*}
\end{theorem}

To prove the Theorem, we need to first characterize the optimal solution of the shifts optimization problem \eqref{problem:minimization_phi}. This is done in following Lemma, which provides the unique solution for this problem.

\begin{lemma}
    Let $(c_t)_{t=0}^{T-1}$ be a sequence of strictly positive numbers, $(h_t)_{t=0}^{T-1}$ be a sequence of nonnegative numbers and let $\varphi_t(\delta)=\sqrt{c_{t}\delta^2+h_{t}}$ be the $t$-th modulus of continuity. Let also $E:\RR{T-1}\to\RR{}$ be defined as in \eqref{problem:minimization_phi} and $\mathbf{u^{\ast}}\in\RR{T-1}$ be recursively defined as:
    \begin{equation*}
        u_t^{\ast}=\left(\frac{\sum_{k=t}^{T-1}\Pi_{l=k+1}^{T-1}c_l\sigma_k^2}{\sum_{j=t-1}^{T-1}\Pi_{l=j+1}^{T-1}c_l\sigma_j^2}\right)\varphi_{t-1}(u_{t-1}),\quad \text{for all}\quad t=1,\ldots,T-1.
    \end{equation*}
    Then $\mathbf{u^{\ast}}\in\RRR$ and is the unique minimizer of $E$ over $\RR{T-1}$.
\end{lemma}

\begin{proof}
First, note that $\mathbf{u^{\ast}}\in\RRR$. This follows from the fact that 
    \begin{equation*}
        \left(\frac{\sum_{k=t}^{T-1}\Pi_{l=k+1}^{T-1}c_l\sigma_k^2}{\sum_{j=t-1}^{T-1}\Pi_{l=j+1}^{T-1}c_l\sigma_j^2}\right)\leq 1\quad (\forall t=1,\ldots,T).
    \end{equation*}
    
    We break the proof of $\mathbf{u^{\ast}}$ being a minimizer into two separate statements: $\mathbf{u^{\ast}}$ is the unique stationary point and $\mathbf{u^{\ast}}$ is the global minimizer.

\noindent\underline{$\mathbf{u^{\ast}}$ is the unique stationary point of $E$:} Computing the partial derivatives of $E$ and arranging the terms, we get that the stationary conditions for $\mathbf{u}$ are
    \begin{equation}\label{eqn:stationary_conditions_general}
        (c_t\sigma_{t-1}^2+\sigma_t^2)u_t-\sigma_{t-1}^2\varphi_{t}^{\prime}(u_t)u_{t+1}=\sigma_t^2\varphi_{t-1}(u_{t-1}),\quad\forall t=1,\ldots,T-1.
    \end{equation}
    We will show by reverse induction that if $\mathbf{u}$ satisfies the \eqref{eqn:stationary_conditions_general}, then $\mathbf{u}=\mathbf{u^{\ast}}$.

    Suppose $\mathbf{u}$ satisfies the stationary conditions. Then, for $t=T-1$, we have that
    \begin{align*}
        u_{T-1}&=\left(\frac{\sigma_{T-1}^2}{c_{T-1}\sigma_{T-2}^2+\sigma_{T-1}^2}\right)\varphi_{T-2}(u_{T-2})\\
        &=u^{\ast}_{T-1}
    \end{align*}

    As induction hypothesis (IH, from now on), suppose that
    \begin{equation*}
        u_{T-s}=u_{T-s}^{\ast}=\left(\frac{\sum_{k=T-s}^{T-1}\Pi_{l=k+1}^{T-1}c_l\sigma_k^2}{\sum_{j=T-(s+1)}^{T-1}\Pi_{l=j+1}^{T-1}c_l\sigma_j^2}\right)\varphi_{T-(s+1)}(u_{T-(s+1)}).
    \end{equation*}
    Plugging IH into \eqref{eqn:stationary_conditions_general} for $t=T-(s+1)$ and reordering terms, we get
    \begin{equation*}
        \sigma_{T-(s+1)}^2\left(\frac{\sum_{k=T-(s+2)}^{T-1}\Pi_{l=k+1}^{T-1}c_l\sigma_k^2}{\sum_{j=T-(s+1)}^{T-1}\Pi_{l=j+1}^{T-1}c_l\sigma_j^2}\right)u_{T-(s+1)}=\sigma_{T-(s+1)}^2\varphi_{T-(s+2)}(u_{T-(s+2)}).
    \end{equation*}
    Therefore,
    \begin{equation*}
        u_{T-(s+1)}=\left(\frac{\sum_{k=T-(s+1)}^{T-1}\Pi_{l=k+1}^{T-1}c_l\sigma_k^2}{\sum_{j=T-(s+2)}^{T-1}\Pi_{l=j+1}^{T-1}c_l\sigma_j^2}\right)\varphi_{T-(s+2)}(u_{T-(s+2)}),
    \end{equation*}
    completing the induction.

\noindent\underline{$\mathbf{u}^{\ast}$ is the global minimizer:} Since $E$ is continuous and nonnegative, there exists a sequence $(\mathbf{x}^{n})_{n\in\mathbb{N}}\subseteq\RR{T-1}$ such that
\begin{equation*}
            \lim_{n\to\infty} E (\mathbf{x}^n)=\inf_{\uu\in\RR{T-1}} E (\uu)>-\infty.
        \end{equation*}
We will prove by reverse induction that the sequences of coordinates of $(\mathbf{x}^{n})_{n\in\mathbb{N}}$ are bounded. To simplify notation, let $x^{n}_T=0$ and $x^{n}_0=D$ for all $n\in\mathbb{N}$. Since
\begin{equation*}
                 E (\mathbf{x}^n)=\sum_{t=1}^{T}\frac{(\varphi_{t-1}(x^{n}_{t-1})-x^{n}_{t})^2}{\sigma_{t-1}^2}\geq \frac{\varphi_{T-1}(x^{n}_{T-1})^2}{\sigma_{T-1}^2},
\end{equation*}
the sequence $(x^{n}_{T-1})_{n\in\mathbb{N}}$ must be bounded. Assume now that the sequence $(x_{T-s}^{n})_{n\in\mathbb{N}}$ is bounded for $s\geq 1$. Then, by the induction hypothesis and the fact that
\begin{equation*}
    E (\mathbf{x}^n)=\sum_{t=1}^{T}\frac{(\varphi_{t-1}(x^{n}_{t-1})-x^{n}_{t})^2}{\sigma_{t-1}^2}\geq \frac{\left(\varphi_{T-(s+1)}(x^{n}_{T-(s+1)})-x^n_{T-s}\right)^2}{\sigma_{T-(s+1)}^2},
\end{equation*}
$(x^{n}_{T-(s+1)})_{n\in\mathbb{N}}$ must also be bounded, which concludes the induction. Since the sequences of coordinates of $(\mathbf{x}^n)_{n\in\mathbb{N}}$ are all bounded, $(\mathbf{x}^n)_{n\in\mathbb{N}}$ is also bounded. Then, by the Bolzano-Weierstrass Theorem, there exists $\mathbf{x^{\ast}}\in\RR{T-1}$ that minimizes $E$, and by first-order conditions $\mathbf{x^{\ast}}=\mathbf{u^{\ast}}$.
\end{proof}

\begin{proof}[Proof of Theorem \ref{cor:minimized_E}]
    Note that the optimal solution $(u_1^{\ast},\ldots,u_{T-1}^{\ast})=\mathbf{u}^{\ast}\in\RRR$ determined in the above Lemma, allow to define a sequence of feasible shifts $a_t^{\ast}=\varphi_{t-1}(u_{t-1}^{\ast})-u_t^{\ast}$ (see Section \ref{section:shifts_optimization_problem}). Hence, by Lemma \ref{lemma:privacy_amp},
    \begin{align*}
        R_{\alpha}(X_T||X_T^{\prime})\leq \frac{\alpha}{2}\sum_{t=1}^{T}\frac{a_t^2}{\sigma_{t-1}^2}
        =\frac{\alpha}{2}E(\mathbf{u^{\ast}})
        =\frac{\alpha}{2}\left(\frac{\Pi_{k=0}^{T-1}c_k D^2}{\sum_{j=0}^{T-1}\sigma_j^2\Pi_{l=j+1}^{T-1}c_l}+\sum_{t=0}^{T-1}\frac{h_t\Pi_{k=t+1}^{T-1}c_k}{\sum_{j=t}^{T-1}\sigma_j^2\Pi_{l=j+1}^{T-1}c_l}\right),
    \end{align*}
    where the last equality follows by a simple evaluation of $E(\mathbf{u^{\ast}})$.
\end{proof}

\begin{remark}
    There are moduli of continuity of interesting problems which are not of the form $\varphi(\delta)=\sqrt{c\delta^2+h}$. One of them comes from the noisy Gradient Descent-Ascent applied to convex/concave saddle-point problems. It can be shown that if the potential and its gradient are $L$-Lipschitz and $\beta$-smooth, respectively, then the modulus of continuity of the iteration is $\varphi(\delta)=\sqrt{\delta^2 +\min\{2\eta L, \eta\beta\delta\}^2}$. We were not able to analyze this modulus. However, it should be noted that it is possible to compare it with that obtained from the convex and Lipschitz potential, which always dominates it. Therefore, one can prove a PABI bound for convex/concave saddle-point problems which is upper bounded by that of the convex Lipschitz case.
\end{remark}

\subsection{Consequences of Theorem \ref{cor:minimized_E}}

In this subsection we give explicit bounds for the different modulus of continuity presented in Section \ref{section:Preliminaries}. In all of the results below, we used constant stepsize, $\eta>0$, and constant variance of the noise added in each iteration, $\sigma^2>0$. The cases $\sigma^2=2\eta$ and $\sigma^2=\eta^2$ are of particular interest for sampling and differential privacy, respectively. Each result is obtained by replacing the respective $c_t$ and $h_t$ in the bound of Theorem \ref{cor:minimized_E}.

\subsubsection{Convex and Lipschitz Potentials, Convex and \texorpdfstring{$(p,M)$}{(p,M)}-Weakly Smooth Potentials}

Recall that when the potential, $f$, is convex and $L$-Lipschitz or convex and $(p,M)$-weakly smooth, their moduli of continuity associated to the gradient maps are of the form $\varphi(\delta)=\sqrt{\delta^2+h}$, where $h =(2\eta L)^2$ when the potential is Lipschitz, and $h\!=\!\!\left(2\eta^{\frac{1}{1-p}}\sqrt{\frac{1-p}{1+p}}(M/2)^{\frac{1}{1-p}}\right)^2$ when the potential is $(p,M)$-weakly smooth. The following Corollary establishes a bound for both cases.

\begin{corollary}\label{cor:renyi_weaklysmooth}
    Let $(X_t)_{t\in\NNo}$ and $(X_t^{\prime})_{t\in\NNo}$ be two projected noisy iterations with modulus of continuity $\varphi(\delta)=\sqrt{\delta^2+h}$, where $h>0$, which only differ in their initialization and whose domain, $\XX$, has diameter $D>0$, then
    \begin{equation*}
        R_{\alpha}(X_T||X_T^{\prime})\leq \frac{\alpha}{2\sigma^2}\left(\frac{D^2}{T}+h\sum_{t=1}^{T}\frac{1}{t}\right)\leq \frac{\alpha}{2\sigma^2}\left(\frac{D^2}{T}+h\ln(T\cdot e)\right).
    \end{equation*}
\end{corollary}

\subsubsection{Strongly Dissipative Potential}

Finally, we state the R\'enyi divergence bound when the potential, $f$, is $(\lambda,\kappa)$-strongly dissipative and $\beta$-smooth.

\begin{corollary}\label{cor:renyi_strdissip}
    Let $\mathcal{X}\subseteq \mathbb{R}^d$ with diameter $D>0$, and $(X_t)_{t\in\NNo}$, $(X_t^{\prime})_{t\in\NNo}$ be two projected noisy iterations with modulus of continuity $\varphi(\delta)=\sqrt{(1-2\eta\kappa+\eta^2\beta^2)\delta^2+2\eta\lambda}$, which only differ in their initialization. Then
    \begin{equation*}
    \begin{aligned}
        R_{\alpha}(X_T||X_T^{\prime})&\leq
        \frac{\alpha}{2\sigma^2}\left(\frac{D^2 c^T(1-c)}{(1-c^T)}+h\sum_{t=0}^{T-1}\frac{c^{t}}{\sum_{j=0}^{t}c^{j}}\right)\\
        &\leq \frac{\alpha}{2\sigma^2}\left(\frac{D^2 c^T(1-c)}{(1-c^T)}+h\ln\left(\left(\frac{1-c^T}{1-c}\right) e\right)\right),
        \end{aligned}
    \end{equation*}
    where $c=1-2\eta \kappa + \eta^2\beta^2$ and $h=2\eta \lambda$.
\end{corollary}

The first inequality is a direct application of Theorem \ref{cor:minimized_E}, while the second is obtained through an integral estimation of the sum. Indeed, by observing that if a function, $F$, is nonnegative and decreasing, then $\sum_{t=1}^{T-1}c^tF\left(\sum_{t=0}^{t} c^t\right)\leq \int_{1}^{\sum_{t=0}^{T-1} c^t}F(x)dx$, we get 
\begin{equation*}
 \begin{aligned}
    1+\ln\left(\frac{1-c^{T+1}}{1-c^2}\right)&\leq\sum_{t=0}^{T-1}\frac{c^{t}}{\sum_{j=0}^{t}c^{j}}= 1+\sum_{t=1}^{T-1}\frac{c^{t}}{1+\sum_{j=1}^{t}c^{j}}\\
    &\leq 1+\int_{1}^{1+\sum_{t=1}^{T-1}c^{t}}\frac{1}{x}\,dx = 1+\ln\left(\frac{1-c^T}{1-c}\right),
     \end{aligned}
\end{equation*}
where we also include a lower bound to show that our integral estimate is nearly tight.

\begin{remark}
    We would like to point out that our general result allows to recover some settings where the PABI approach has been applied. More precisely, by setting $h\equiv 0$, $c$ and $\sigma$ constant in Theorem \ref{cor:minimized_E}, we can recover the contractive, nonexpansive and expansive results obtained in \citet{feldman2018privacy} and Remark A.4 of \citet{altschuler2022resolving}. This corresponds, respectively, to the cases where the potential $f$ is  $\kappa$-strongly convex and $\beta$-smooth, convex and $\beta$-smooth and nonconvex and $\beta$-smooth.
\end{remark}
%%%%%%%%%%% Mixing times %%%%%%%%%%%%%%%%%%%%%%%

\section{Mixing Time Bounds for the Projected Langevin Algorithm}\label{section:mixing_time}

In this section we study the mixing time in total variation distance for the Projected Langevin Algorithm (PLA). Since the result only relies on the bound obtained in Section \ref{section:PABI}, it also holds for (possibly nondifferentiable) potentials $f$ that are convex and $L$-Lipschitz, by replacing $p=0$ and $M=2L$. The problem at hand is not only of academic interest, finding applications in Bayesian inference. An interesting example involves sampling from potentials of the form $\exp(-\|Ax-b\|_2^2-\|B x\|_{p+1}^{p+1})$, where $0\leq p\leq 1$. This model arises from considering a prior distribution on hypotheses given by a linear transformation $B$ of an $\ell^{p+1}$-ball, and its posterior resulting from linear observations, determined by input data matrix $A$ and corresponding output vector $b$ (see \citealt{Chatterji:2020} for further discussions).

Recall that PLA is \eqref{eqn:langevin_dynamic}, with $\sigma=\sqrt{2\eta}$. That is

\begin{equation}\tag{PLA}\label{eqn:projected_langevin_algorithm_mixing}
    X_{t+1} = \Pi_{\XX}\left[X_t -\eta \nabla f(X_t)+\sqrt{2\eta}\xi_{t}\right],
\end{equation}
where $(\xi_t)_{t\in\mathbb{N}_0}\overset{i.i.d.}{\sim}\NN(0,I_{d\times d})$.

Notice that \eqref{eqn:projected_langevin_algorithm_mixing} is a homogeneous Markov chain (henceforth HMC), whose only involved map (apart from the noise addition) is the gradient mapping $\Phi=I-\eta\nabla f$.

\subsection{Convex and Weakly Smooth Case}
Based on our settings of interest, we assume $\Phi$ has a modulus of continuity  $\varphi(\delta)=\sqrt{\delta^2+h}$ and that $\sigma^2=2\eta$. From Corollary \ref{cor:renyi_weaklysmooth} and taking $\alpha=1$, we obtain the KL bound

\begin{equation}\label{eqn:kl_bound_mixing}
    \text{KL}(X_T||X_T^{\prime})\leq \underbrace{\frac{D^2}{4\eta T}}_{\Circled{I}}+\underbrace{\frac{h\ln(T\cdot e)}{4\eta}}_{\Circled{II}}.
\end{equation}

Observe that $\Circled{I}$ in \eqref{eqn:kl_bound_mixing} is exactly the bound obtained for the nonexpansive case in \citet[Proposition 2.10]{altschuler2022resolving}.
Consequently, the price required for utilizing moduli of continuity of the form $\varphi(\delta)=\sqrt{\delta^2+h}$ is encapsulated by the term added in $\Circled{II}$.

Replacing $h=\left(2\eta^{\frac{1}{1-p}}\sqrt{\frac{1-p}{1+p}}(M/2)^{\frac{1}{1-p}}\right)^2$ in equation \eqref{eqn:kl_bound_mixing}, we get
\begin{equation*}
    \text{KL}(X_T||X_T^{\prime})\leq \frac{D^2}{4\eta T}+\ln(T\cdot e)\left(\eta^{\frac{1+p}{1-p}}\left(\frac{1-p}{1+p}\right)(M/2)^{\frac{2}{1-p}}\right).
\end{equation*}

Since a given potential can be $(p,M)$-weakly smooth with multiple parameters, the bound also is satisfied with the infimum; that is
\begin{equation*}
    \text{KL}(X_T||X_T^{\prime})\leq \frac{D^2}{4\eta T}+\ln(T\cdot e)\cdot \inf\left\{\left(\eta^{\frac{1+p}{1-p}}\mbox{$\left(\frac{1-p}{1+p}\right)$}(M/2)^{\frac{2}{1-p}}\right)\!:\! f\text{ is }(p,M)\text{-weakly smooth}\right\}.
\end{equation*}
Notice that $M(p)=\inf\{M>0: f \mbox{ is $(p,M)$-weakly smooth}\}$ is a log-convex function with respect to $p$; therefore, the infimum above may have a nontrivial optimal choice of $p$. In this regard, it is interesting that we can automatically obtain this adaptivity in terms of $p$.

We use the obtained KL bound \eqref{eqn:kl_bound_mixing} to establish a new mixing time for the projected Langevin algorithm in the weakly smooth and Lipschitz convex cases.  We remark that, aside from a slightly more restrictive range of stepsize parameters, the mixing time bound is entirely analogous to that of the smooth convex case \citep{altschuler2022resolving}. Since it is of practical interest to set the stepsize sufficiently small, the stepsize restriction is not problematic. Note that the result only uses the modulus of continuity of the potential, so it also holds for the $L$-Lipschitz setting replacing $p=0$ and $M=2L$.

We will start providing a total variation mixing time (see Definition \ref{def:mixing_time}) to constant error $1/2$.

\begin{lemma}\label{thm:contraction_of_iterates_holder}
    Let $\XX\subseteq\RR{d}$ be a convex, compact set with diameter $D>0$ and suppose that $f:\XX\to\RR{}$ is a convex and $(p,M)$-weakly smooth function, with $p\in[0,1]$ and $M>0$. Let $(X_t)_{t\in\NNo}$ and $(X_t^{\prime})_{t\in\NNo}$ be two HMC generated by \eqref{eqn:projected_langevin_algorithm_mixing} that only differ in their initialization. Let 
    \begin{equation}\label{eqn:bd_eta}
    \textstyle\Theta=\left(\frac{M}{2}\right)^{\left(\frac{2}{1+p}\right)}\left[\left(\frac{1-p}{1+p}\right)\max\left\{16\lnn{D\left(\frac{M}{2}\right)^{\frac{1}{1+p}}e},27\right\}\right]^{\left(\frac{1-p}{1+p}\right)}.
    \end{equation}
    If $\frac{1}{\eta}\geq \Theta$ 
    and $T=\left\lceil\frac{D^2}{\eta}\right\rceil$, then $\normtv{\proba_{X_T}-\proba_{X_T^{\prime}}}\leq \frac{1}{2}.$
\end{lemma}

\begin{proof}
    To simplify notation, let us call $\Tilde{\eta}:=\eta^{\frac{1+p}{1-p}}$ and $\Tilde{M}:=\sqrt{\frac{1-p}{1+p}}(M/2)^{\frac{1}{1-p}}$.

    By equation \eqref{eqn:kl_bound_mixing} and Proposition \ref{prop:pinskers_inequality} it is enough to find $T$ and $\eta$ such that $\frac{D^2}{4\eta T}+\Tilde{\eta}\Tilde{M}^2\lnn{T\cdot e}\leq 1/2$. One way to get this bound is to derive individually:
    \begin{align*}
        \frac{D^2}{4\eta T}&\leq \frac{1}{4} \qquad
        \text{ and}\qquad
        \Tilde{\eta} \Tilde{M}^2\ln(T\cdot e)\leq \frac{1}{4}.
    \end{align*}
    For the first inequality it is enough to take $T=\left\lceil\frac{D^2}{\eta}\right\rceil$. For the second inequality, we plug $\frac{2D^2}{\eta}$ in the place of $T$, which is a sufficient condition for the inequality when $\eta\leq D^2$. Then, after some simple algebraic manipulation,  we get that
\begin{equation}\label{eqn:determining_eta_holder}
        2\lnn{\sqrt{2e}D\Tilde{M}^{\frac{1-p}{1+p}}}\leq \frac{1}{4\Tilde{\eta}\Tilde{M}^2}-\left(\frac{1-p}{1+p}\right)\lnn{\frac{1}{\Tilde{\eta}\Tilde{M}^2}}.
    \end{equation}
    The right side of \eqref{eqn:determining_eta_holder} can be written as a function of $1/(\Tilde{\eta}\Tilde{M}^2)$, namely $F_{p}(x)=\frac{x}{4}-\left(\frac{1-p}{1+p}\right)\ln(x)$. Since $p\in [0,1]$, $F_p$ can be lower bounded by $F(x)=\frac{x}{4}-\lnn{x}$ when $\eta^{-1}\geq \left(\frac{1-p}{1+p}\right)^{\frac{1-p}{1+p}}\left(\frac{M}{2}\right)^{\frac{2}{1+p}}$. It can be shown that $F(x)\geq\frac{x}{8}$ when $x\geq 27$. Thus, a sufficient condition for \eqref{eqn:determining_eta_holder} to hold is to ask for $2\lnn{\sqrt{2e}D\Tilde{M}^{\frac{1-p}{1+p}}}\leq\frac{1}{8\Tilde{\eta}\Tilde{M}^2}$,
    subject to $\frac{1}{\Tilde{\eta}\Tilde{M}^2}\geq 27$. Therefore, a sufficient condition for \eqref{eqn:determining_eta_holder} to hold is:
    \begin{equation*}
        \frac{1}{\Tilde{\eta}}\geq \Tilde{M}^2\max\left\{16\lnn{\sqrt{2e}D\Tilde{M}^{\frac{1-p}{1+p}}},27\right\},
    \end{equation*}

    which is equivalent to
    \begin{equation*}
        \frac{1}{\eta}\geq \left(\frac{M}{2}\right)^{\left(\frac{2}{1+p}\right)}\left[\left(\frac{1-p}{1+p}\right)\max\left\{16\ln\left(\sqrt{2e}D\left(\frac{1-p}{1+p}\right)^{\frac{1-p}{2(1+p)}}\left(\frac{M}{2}\right)^{\frac{1}{1+p}}\right),27\right\}\right]^{\left(\frac{1-p}{1+p}\right)}.
    \end{equation*}
    Finally, noting that
    \begin{equation*}
        \lnn{\sqrt{2e}\left(\frac{1-p}{1+p}\right)^{\frac{1-p}{2(1+p)}}}\leq 1\qquad (\forall p\in [0,1]),
    \end{equation*}
    we obtain the result.
\end{proof}

We highlight the two extreme cases: $p=0$ and $p=1$. When $p=0$ we are in the Lipschitz case and the restriction over $\eta$ boils down to $1/\eta \geq (M/2)^{2}\max\left\{16\lnn{D(M/2)e},27\right\}$. On the other hand, when $p=1$ (which can be obtained through a limit) we are in the smooth case and we recover the restriction $1/\eta\geq M/2$ which makes the gradient mapping nonexpansive.

    A direct consequence of Lemma \ref{thm:contraction_of_iterates_holder} is that for weakly smooth potentials:
    \begin{equation}\label{eqn:mixing_up_to_constant_error}
        T_{mix,TV}(1/2)\leq\left\lceil\frac{D^2}{\eta}\right\rceil,
    \end{equation}
    when its restriction over $\eta$ is satisfied. This can be proved by letting $X_0^{\prime}$ follow the stationary distribution of \eqref{eqn:projected_langevin_algorithm_mixing} (for a proof of existence of stationary distributions of the HMC defined by  \ref{eqn:projected_langevin_algorithm_mixing} in the nondifferentiable case see Appendix \ref{appendix:existence_of_stationary_distributions}). It has been established in \citet[Theorem 3.2]{altschuler2022resolving} that \eqref{eqn:mixing_up_to_constant_error} is tight up to constants by a constant potential, 
    and thus tightness  applies to the above case as well.

Using Lemma \ref{thm:contraction_of_iterates_holder} and a well-known boosting argument (Proposition \ref{prop:appendix_boosting}), we can convert a constant error in total variation into an arbitrary one at a polylogarithmic cost in the accuracy.

\begin{theorem}[Mixing for weakly smooth functions]\label{thm:mixing_for_weakly_smooth}
    Let $\XX\subseteq\RR{d}$ be a convex, compact set with diameter $D>0$ and suppose that $f:\XX\to\RR{}$ is a convex and $(p,M)$-weakly smooth function. If  $1/\eta\geq \Theta$, where $\Theta $ is as in \eqref{eqn:bd_eta}, then, for all $\eps>0$, 
    $$T_{mix,TV}(\eps)\leq \left\lceil\frac{D^2}{\eta}\right\rceil\cdot\lceil\log_2(1/\eps)\rceil.$$
\end{theorem}

    Similarly to \citep{altschuler2022privacy}, when $f=\sum_{i=1}^{n}f_i$ and each $f_i$ is $(p,M)$-weakly smooth, we can replace the use of gradients in \eqref{eqn:projected_langevin_algorithm_mixing} by stochastic (formed through minibatches) gradients. This change yields the same result as in Theorem \ref{thm:mixing_for_weakly_smooth}.

\subsection{Smooth and Strongly Dissipative Case}

We can also establish a mixing time bound for the $\beta$-smooth and $(\lambda,\kappa)$-strongly dissipative case. The analysis is analogous to the one presented before, only that instead of using the Pinsker inequality to derive a $1/2$ bound to which the boosting argument is applied, we use the Bretagnolle-Huber inequality (Proposition \ref{lemma:Bretagnole-Huber_inequality}), which is valid in a broader regime. The boosting argument is still applicable, but at the cost of a worst dependence on the parameters.

\begin{theorem}\label{thm:mixing2}
    Let $\XX\subseteq\RR{d}$ be a convex, compact set with diameter $D>0$ and suppose that $f:\XX\to\RR{}$ is a $(\lambda,\kappa)$-strongly dissipative and $\beta$-smooth function. Let $c=1-2\eta\kappa +\eta^2\beta^2$. Then, for all $\eps>0$, $$T_{mix,TV}(\eps)\leq \left\lceil \log_{1/c}\left(1+\frac{D^2(1-c)}{4\eta}\right)\right\rceil\cdot \left\lceil 2e\ln(2)\left(\frac{e}{1-c}\right)^{\lambda/2}\log_2(1/\eps)\right\rceil.$$
\end{theorem}

\begin{proof} By Corollary \ref{cor:renyi_strdissip}, with $\alpha=1$ and $\sigma^2=2\eta$,
\begin{equation*}
    \text{KL}(X_T||X_T^{\prime})\leq \frac{D^2}{4\eta}\cdot \frac{c^T(1-c)}{1-c^T}+\frac{\lambda}{2}\ln\left(\left(\frac{1-c^T}{1-c}\right)e\right).
\end{equation*}

Taking $T^{\ast}=\left\lceil\log_{1/c}\left(1+\frac{D^2(1-c)}{4\eta}\right)\right\rceil$, we get
\begin{equation}\label{ineq:KL_dissip_Tast}
    \text{KL}(X_{T^{\ast}}||X_{T^{\ast}}^{\prime})\leq 1+\frac{\lambda}{2}\ln\left(\frac{e}{1-c}\right).
\end{equation}

Using Proposition \ref{lemma:Bretagnole-Huber_inequality}, we convert this KL inequality into a total variation one,
\begin{equation*}
    \normtv{\mathbb{P}_{X_{T^{\ast}}}-\mathbb{P}_{X_{T^{\ast}}^{\prime}}}\leq \sqrt{1-\expo{-\text{KL}(X_{T^{\ast}}||X_{T^{\ast}}^{\prime})}}=:\gamma
\end{equation*}

Let $\eps>0$. In order to get $\gamma^R\leq \eps$, we need $R\geq \left\lceil \log_{1/\gamma}(1/\eps)\right\rceil$.

Notice that
\begin{equation}\label{eqn:gamma_usingbase2}
    \log_{1/\gamma}(1/\eps)=\frac{\ln(2)\log_2(1/\eps)}{\ln(1/\gamma)}
\end{equation}
and that by the convexity of $-\ln(1-x)$
\begin{align}
    \ln(1/\gamma)&=-\frac{1}{2}\ln\left(1-e^{-\text{KL}(X_{T^{\ast}}||X_{T^{\ast}}^{\prime})}\right)\notag\\
    &\geq \frac{1}{2}e^{-\text{KL}(X_{T^{\ast}}||X_{T^{\ast}}^{\prime})}\label{ineq:log_kl}.
\end{align}

Hence, by a boosting argument of the total variation \citep[Lemma 4.11]{levin2017markov}, we get $T_{mix,TV}(\eps)\leq T^{\ast}\cdot R$. Finally, using the definition of $T^{\ast}$, \eqref{ineq:KL_dissip_Tast}, \eqref{eqn:gamma_usingbase2} and \eqref{ineq:log_kl}, we conclude
\begin{equation*}
    T_{mix,TV}(\eps)\leq \left\lceil \log_{1/c}\left(1+\frac{D^2(1-c)}{4\eta}\right)\right\rceil\cdot \left\lceil 2e\ln(2)\left(\frac{e}{1-c}\right)^{\lambda/2}\log_2(1/\eps)\right\rceil.
\end{equation*}
\end{proof}

%%%%%%%%%%% Privacy analysis %%%%%%%%%%%%%%%%%%%%%%%

\section{Privacy Analysis of Noisy SGD}\label{section:privacy_analysis}

In this section, we combine the analysis from \citet{altschuler2022privacy} with the PABI bounds from Section \ref{section:PABI} to derive new privacy bounds for noisy SGD's last iteration.

We remind the reader of the definition of differential privacy. We denote a data set by an $n$-tuple $S=(z_1,\ldots,z_n)\in{\ZZ}^n$, where $\ZZ$ is the data space. First, we say that two data sets $S,S'$ are neighbours, denoted by $S\simeq S'$, if they only differ in one of their entries.

\begin{definition}[Differential Privacy] A randomized algorithm ${\AAA}:{\ZZ}^n\mapsto{\XX}$ is $(\varepsilon,\delta)$-differentially private (DP) if for every pair of data sets $S\simeq S'$, and any event $O\subseteq \XX$,
\[ \mathbb{P}[{\AAA}(S)\in O] \leq \exp(\varepsilon) \mathbb{P}[{\AAA}(S')\in O]+\delta. \]
\end{definition}

As mentioned above, we will analyze the privacy curve (i.e.~a bound of the R\'enyi divergence of two outputs of the same algorithm when executed on two neighboring data sets) of noisy SGD, a data dependent algorithm that, given a data set $S=(z_1,\ldots,z_n)$ and an initializaton $X_0\in\XX$:
\begin{enumerate}
    \item To update its state at time $t\in\{1,\ldots,T\}$:
    \begin{enumerate}
        \item[$(i)$] Using Poisson sampling, randomly chooses a minibatch $B_t\subseteq\{1,\ldots,n\}$ of expected size $b$ (i.e. each data point $z_i$ has probability $b/n$ of being in $B_t$).
        \item[$(ii)$] Given $\xi_t \sim\NN(0,\eta^2\sigma^2 I_{d\times d})$, computes 
            $X_{t+1}=\Pi_{\XX}\left[X_t-\frac{\eta}{b}\sum_{i\in B_t}\nabla f(X_t,z_i)+\xi_t\right].$
    \end{enumerate}
    \item Return $X_T$.
\end{enumerate}

 We study the case where the loss functions $f(\cdot,z)$ are convex, $L$-Lipschitz and $(p,M)$-weakly smooth for every $z\in{\ZZ}$, where $p\in [0,1]$, $M>0$ and ${\ZZ}$ is a data space.

\begin{theorem}\label{thm:privacy_analysis_nsgd}
    Let $\XX\subseteq\RR{d}$ be a convex and compact set with diameter $D>0$. For any number of iterations $T>\overline{T}=\left\lceil \frac{Dn}{4\eta L}\right\rceil$, datasize $n\in\mathbb{N}$, expected batch size $b\leq n$, stepsize $\eta>0$, initialization $x_0\in\XX$ and noise parameter $\sigma>8\sqrt{2}L/b$, noisy SGD applied to convex, $L$-Lipschitz and $(p,M)$-weakly smooth losses satisfies $(\alpha,\eps)$-RDP for $1<\alpha\leq\alpha^{\ast}\left(\frac{b}{n},\frac{b\sigma}{2\sqrt{2}L}\right)$ and:
    \begin{equation*}
        \eps\leq
        \frac{\alpha}{\sigma^2}\left(\frac{16 L^2 \overline{T}}{n^2}+\frac{D^2}{\eta^2\overline{T}}+4\eta^{\frac{2p}{1-p}}\left(\frac{1-p}{1+p}\right)\left(\frac{M}{2}\right)^{\frac{2}{1-p}}\ln\left(\overline{T}\cdot e\right)\right)
    \end{equation*}
\end{theorem}

\begin{proof}
    The proof is an analogue of \citet[Theorem 3.1]{altschuler2022privacy}. Let $S,S^{\prime}\in{\ZZ}^{n}$ be two neighbor data sets that differs in the data point corresponding to $i^{\ast}$, that is, $  z_i=z_i^{\prime}$ for all $ i\neq i^{\ast}$. Run noisy SGD on both data sets, $S$ and $S^{\prime}$, for $T$ iterations and call the respective trajectories:
    \begin{align*}
        X_{t+1}&=\Pi_{\XX}\left[X_t-\frac{\eta}{b}\sum_{i\in B_t}\nabla f(X_t,z_i)+\xi_t\right]\\
        X_{t+1}^{\prime}&=\Pi_{\XX}\left[X_t^{\prime}-\frac{\eta}{b}\sum_{i\in B_t}\nabla f(X_t^{\prime},z_i^{\prime})+\xi_t\right].
    \end{align*}
    This trajectories start from the same point where the noise injection $(\xi_t)_{t=0}^{T-1}$ and minibatch $(B_t)_{t=0}^{T-1}$ are coupled. One can rewrite this expressions as
    \begin{align*}
        X_{t+1}&=\Pi_{\XX}\left[X_t-\frac{\eta}{b}\sum_{i\in B_t}\nabla f(X_t,z_i)+Y_t+Z_t\right]\\
        X_{t+1}^{\prime}&=\Pi_{\XX}\left[X_t^{\prime}-\frac{\eta}{b}\sum_{i\in B_t}\nabla f(X_t^{\prime},z_i)+Y_t+Z_t^{\prime}\right],
    \end{align*}
    where $Y_t\sim\NN(0,(\eta^2\sigma^2/2) I_{d\times d})$, $Z_t\sim\NN(0,(\eta^2\sigma^2/2) I_{d\times d})$ and
     $$Z_t^{\prime}\sim\NN\Big(\frac{\eta}{b}\left[\nabla f(X_t^{\prime},z_{i^{\ast}})-\nabla f(X_t^{\prime},z_i)\right]\cdot\mathbbm{1}_{\{i^{\ast}\in B_t\}},\,(\eta^2\sigma^2/2) I_{d\times d}\Big).$$
    It is important to remark that the gradients of the convex losses that we are using in both trajectories come from the data set $S$, not $S^{\prime}$. Observe also that the bias term is realized with probability 
\begin{equation}\label{eqn:probability_of_index_in_batch}
            \prob{i^{\ast}\in B_t}=\frac{b}{n}.
        \end{equation}

    Conditional on the event that $Z_t=Z_t^{\prime}$ (call $z_t$ its realization):
    \begin{align*}
        X_{t+1}&=\Pi_{\XX}\left[\Phi_t(X_t)+Y_t \right]\\
        X^{\prime}_{t+1}&=\Pi_{\XX}\left[\Phi_{t}(X_t^{\prime})+Y_t\right],
    \end{align*}
    where
    \begin{equation}\label{eqn:Psi_privacy_analysis}
        \Phi_{t}(x):=x-\frac{\eta}{b}\sum_{i\in B_t}\nabla f_i(x)+z_t.
    \end{equation}
    
    Note that the modulus of continuity of \eqref{eqn:Psi_privacy_analysis} is upper bounded by the modulus of continuity of the noiseless gradient mapping. This leads to a modulus of continuity $\varphi(\delta)=\sqrt{\delta^2+h}$ for $h=\left(2\eta^{\frac{1}{1-p}}\sqrt{\frac{1-p}{1+p}}\left(\frac{M}{2}\right)^{\frac{1}{1-p}}\right)^2$.

    Conditional on the event $Z_t=Z_t^{\prime}$ for all $t\geq\tau$, the processes $\{X_t\}_{t\geq \tau}$ and $\{X_t^{\prime}\}_{t\geq \tau}$ are projected noisy iterations with modulus of continuity $\varphi$, where $\tau\in\{0,\ldots,T-1\}$ is a parameter chosen {\em a posteriori}.

    The bound of $R_{\alpha}(\proba_{X_T}||\proba_{X_T^{\prime}})$ is obtained through Privacy Amplification by Sampling and Privacy Amplification by Iteration (with modulus of continuity): 
    \begin{align}
        R_{\alpha}\left(\proba_{X_T}||\proba_{X_T^{\prime}}\right)&\leq
        R_{\alpha}\left(\proba_{X_T,Z_{\tau:T}}||\proba_{X_T^{\prime},Z_{\tau:T}^{\prime}}\right)\nonumber\\
        &\leq
    \underbrace{R_{\alpha}\left(\proba_{Z_{\tau:T-1}}||\proba_{Z_{\tau:T-1}^{\prime}}\right)}_{\Circled{1}}+\underbrace{\sup_{z}R_{\alpha}\left(\proba_{X_T|Z_{\tau:T-1}=z}||\proba_{X_T^{\prime}|Z_{\tau:T-1}^{\prime}=z}\right)}_{\Circled{2}}\label{eqn:splitting_renyi_divergence_privacy_analysis},
    \end{align}
    where the first line follows from the data-processing inequality (Proposition \ref{prop:data_processing_ineq}) and the second from strong composition (Proposition \ref{prop:strong_composition}).

    \underline{Bounding \Circled{1} through Privacy Amplification by Sampling:} 

    \begin{align}
        \Circled{1}
        &\leq
        \sum_{t=\tau}^{T-1}\sup_{z_{\tau:t-1}}R_{\alpha}\left(\proba_{Z_t|Z_{\tau:t-1}=z_{\tau:t-1}}||\proba_{Z_t^{\prime}|Z_{\tau:t-1}^{\prime}=z_{\tau:t-1}}\right)\nonumber\\
        &=\sum_{t=\tau}^{T-1}R_{\alpha}\left(\NN\left(0,\frac{\eta^2\sigma^2}{2}I_{d\times d}\right){\Big | \Big |}\left(1-\frac{b}{n}\right)\NN\left(0,\frac{\eta^2\sigma^2}{2}I_{d\times d}\right)+\frac{b}{n}\NN\left(m_t,\frac{\eta^2\sigma^2}{2}I_{d\times d}\right)\right)\nonumber\\
        &\leq (T-\tau)S_{\alpha}\left(\frac{b}{n},\frac{b\sigma}{2\sqrt{2}L}\right)\label{eqn:bound_of_1_privacy_analysis},
    \end{align}
    where the first line follows from strong composition (Proposition \ref{prop:strong_composition}). The second line follows by the independence of the $(Z_t)_t$: $Z_t\sim\NN\left(0,\frac{\eta^2\sigma^2}{2}I_{d\times d}\right)$ conditioned on $Z_{\tau:t-1}=z_{\tau:t-1}$ for any $z_{\tau:t-1}$; also, by \eqref{eqn:probability_of_index_in_batch} and the independence of the $(Z_t^{\prime})$, the law of $Z_t^{\prime}$ is the mixture of $\NN\left(0,\frac{\eta^2\sigma^2}{2}I_{d\times d}\right)$ and $\NN\left(m_t,\frac{\eta^2\sigma^2}{2}I_{d\times d}\right)$, where $m_t:=\frac{\eta}{b}\left[\nabla f(X_t^{\prime},z_{i^{\ast}})-\nabla f(X_t,z_{i^{\ast}})\right]$. The last line follows from the fact that $\norm{m_t}\leq 2\eta L/b$ and the bound in Lemma \ref{lemma:bound_mixture}.

\vspace{1ex}
    \underline{ Bounding \Circled{2} through Corollary \ref{cor:renyi_weaklysmooth}:} As we already mentioned, conditional on the event that $Z_t=Z_t^{\prime}$ for all $t\geq\tau$, the sequences $\{X_t\}_{t\geq\tau}$ and $\{X_t^{\prime}\}_{t\geq\tau}$ are projected noisy iterations with moduli of continuity $\varphi(\delta)=\sqrt{\delta^2+h}$. Then, by Corollary \ref{cor:renyi_weaklysmooth}, using $h=\left(2\eta^{\frac{1}{1-p}}\sqrt{\frac{1-p}{1+p}}\left(\frac{M}{2}\right)^{\frac{1}{1-p}}\right)^2$ for all $t\geq\tau$ and $\sigma_j^2=\frac{\eta^2\sigma^2}{2}$ for all $j\geq\tau$:  
    \begin{align}
        \Circled{2}
        &\leq \frac{\alpha D^2}{\eta^2\sigma^2(T-\tau)}+\frac{\alpha \left(2\eta^{\frac{1}{1-p}}\sqrt{\frac{1-p}{1+p}}\left(\frac{M}{2}\right)^{\frac{1}{1-p}}\right)^2}{\eta^2\sigma^2}\ln{\left((T-\tau)\cdot e\right)}\nonumber\\
        &\leq \frac{\alpha D^2}{\eta^2\sigma^2(T-\tau)}+\frac{4\alpha \eta^{\frac{2p}{1-p}}}{\sigma^2}\left(\frac{1-p}{1+p}\right)\left(\frac{M}{2}\right)^{\frac{2}{1-p}}\ln{\left((T-\tau)\cdot e\right)}\label{eqn:bound_of_2_privacy_analysis}.
    \end{align}
    Plugging \eqref{eqn:bound_of_1_privacy_analysis} and \eqref{eqn:bound_of_2_privacy_analysis} into \eqref{eqn:splitting_renyi_divergence_privacy_analysis}, we obtain that noisy SGD is $(\alpha,\eps)$-RDP with:
    \begin{align*}
        \eps\leq\min_{\tau\in\{0,\ldots,T-1\}}\left\{(T-\tau)S_{\alpha}\left(\frac{b}{n},\frac{b\sigma}{2\sqrt{2}L}\right)\right.&+\frac{\alpha D^2}{\eta^2\sigma^2(T-\tau)}+\\
        &\left.+\frac{4\alpha \eta^{\frac{2p}{1-p}}}{\sigma^2}\left(\frac{1-p}{1+p}\right)\left(\frac{M}{2}\right)^{\frac{2}{1-p}}\ln{\left((T-\tau)\cdot e\right)}\right\}
    \end{align*}

    By Lemma \ref{lemma:bound_of_S_alpha}, for all $1<\alpha\leq \alpha^{\ast}\left(\frac{b}{n},\frac{b\sigma}{2\sqrt{2}L}\right)$ and $\sigma\geq 8\sqrt{2}L/b$:
    \begin{equation*}
        S_{\alpha}\left(\frac{b}{n},\frac{b\sigma}{2\sqrt{2}L}\right)\leq \frac{16\alpha L^2}{n^2\sigma^2}.
    \end{equation*}
    Then
    \begin{align*}
        \eps
        &\leq \frac{\alpha}{\sigma^2}\min_{\tau\in\{0,\ldots,T-1\}}\left\{(T-\tau)\frac{16 L^2}{n^2}+\frac{D^2}{\eta^2(T-\tau)}+4 \eta^{\frac{2p}{1-p}}\textstyle{\left(\frac{1-p}{1+p}\right)}\left(M/2\right)^{\frac{2}{1-p}}\ln{\left((T-\tau)\cdot e\right)}\right\}.
    \end{align*}
    One can easily optimize the first two terms of the above expression by naming $R=T-\tau$ and differentiating with respect to $R$. Taking the ceiling of the optimal value for $R$, one obtains:
    \begin{equation*}
        T-\tau=\overline{T}=\left\lceil \frac{Dn}{4\eta L}\right\rceil,
    \end{equation*}
    whenever $T\geq \overline{T}$.
    Therefore,
    \begin{align*}
        \eps&\leq
        \frac{\alpha}{\sigma^2}\left(\frac{16 L^2 \overline{T}}{n^2}+\frac{D^2}{\eta^2\overline{T}}+4\eta^{\frac{2p}{1-p}}\left(\frac{1-p}{1+p}\right)\left(\frac{M}{2}\right)^{\frac{2}{1-p}}\ln\left(\overline{T}\cdot e\right)\right).
    \end{align*}
\end{proof}

\begin{remark}
    Note that by only using Lemmas \ref{lemma:bound_mixture} and \ref{lemma:bound_of_S_alpha} in conjunction with sequential composition for RDP \citep[Proposition 1]{mironov2017renyi} in the privacy analysis of noisy SGD, one obtains that
        $\eps\leq\frac{16\alpha L^2}{n^2\sigma^2}T.$ 
    Therefore, from Theorem \ref{thm:privacy_analysis_nsgd}, we conclude that the last iterate of noisy SGD is $(\alpha,\varepsilon)$-RDP with
    \begin{equation*}
        \eps\leq\frac{16\alpha L^2}{n^2\sigma^2}\min\Big\{T,2\overline{T}+\underbrace{\left(\frac{2\overline{T}}{D}\cdot\left(\frac{\eta M}{2}\right)^{\frac{1}{1-p}}\right)^2\left(\frac{1-p}{1+p}\right)\ln\left(\overline{T}\cdot e\right)}_{V(D,M,\overline{T},\eta,p)}\Big\}
    \end{equation*}

    In comparison to \citet[Theorem 3.1]{altschuler2022privacy}, which holds for smooth functions, working with H\"older continuous gradients adds an extra term to the privacy bound, which we denote above by $V(D,M,\overline{T},\eta,p)$.  See Figure \ref{figure:privacy_curve_caps} for some exemplary plots of the privacy curve. Note that
    in this figure we omit the graph of the case $p=0.8$, as it is indistinguishable from that of $p=1$. Moreover, it appears that for any $p\geq 0.7$ there are no significant differences in the privacy curve bounds with that of the smooth case.
    On the other hand, in the Lipschitz case ($p=0$), it is never possible to obtain a bound that vanishes with $n\to\infty$ since $V(D,M,\overline{T},\eta,0)$ grows as $\Tilde{O}(n^2)$.
\end{remark}

\begin{figure}
\begin{center}
\begin{tikzpicture}[scale=1.2]
    \begin{semilogyaxis}[legend style={at={(axis cs:4,15)},anchor=south west,every axis/.append style={font=\tiny}, legend columns=-1} ,
        table/header=false,
        table/row sep=\\,
        xtick=\empty,
        domain=0:100,
       xmin= -1.0,
      xmax=100,
       ymin= 7,
      ymax=15,
     axis line style = thick,
      axis lines = middle,
      axis x line=bottom,
      enlargelimits = true,
      clip=false,
      ytick={},
    ]
   \node[below right] at (axis cs:100,6.5) {{\small $ \eta$}};

  %%%% p=0.2
\addplot [mark=none,dashdotted,thick] table[x expr=\coordindex,y index=0]{12.479428812939016 \\ 12.361211088280172 \\ 12.355834260402853 \\ 12.381338963730204 \\ 12.416075136122164 \\ 12.452758258425073 \\ 12.488867824584672 \\ 12.52349215406376 \\ 12.556054209662184 \\ 12.586812283344528 \\ 12.615671721872529 \\ 12.642838613833137 \\ 12.66857596717274 \\ 12.69274973374037 \\ 12.715539164437226 \\ 12.73748436677196 \\ 12.75784688964374 \\ 12.777545456637291 \\ 12.79627408795614 \\ 12.814405216533096 \\ 12.831527015187056 \\ 12.84773481615522 \\ 12.863843278516397 \\ 12.879054402356276 \\ 12.89381902289708 \\ 12.907950217665196 \\ 12.921089521324728 \\ 12.934488935453636 \\ 12.947023243518284 \\ 12.95985696495631 \\ 12.971607270854244 \\ 12.983351339699658 \\ 12.994149343430188 \\ 13.00505537398262 \\ 13.015721849280697 \\ 13.026479562008232 \\ 13.036087320943397 \\ 13.046360698150384 \\ 13.055972023920244 \\ 13.065121814945185 \\ 13.074010009220297 \\ 13.082836126527026 \\ 13.091799255233578 \\ 13.100163186823764 \\ 13.10806221199625 \\ 13.116607977887991 \\ 13.124999101367797 \\ 13.1323489310445 \\ 13.13976793785054 \\ 13.147389026819097 \\ 13.15534468778032 \\ 13.16155636368207 \\ 13.169408734193505 \\ 13.1756530811013 \\ 13.182632309713371 \\ 13.189287062496636 \\ 13.19568475181253 \\ 13.20312624038406 \\ 13.209232895573344 \\ 13.215283819169017 \\ 13.221345870854536 \\ 13.227485787140166 \\ 13.233770144880427 \\ 13.23890388169036 \\ 13.244272185643704 \\ 13.249941572882586 \\ 13.2559782791725 \\ 13.261001016158117 \\ 13.266480578010304 \\ 13.272482913112883 \\ 13.277562512825002 \\ 13.283253832215063 \\ 13.288068723243931 \\ 13.293584043804357 \\ 13.298268903024164 \\ 13.3021244647376 \\ 13.308431689935553 \\ 13.312337234525438 \\ 13.317142844203884 \\ 13.321210173910345 \\ 13.326265958469683 \\ 13.330628659745573 \\ 13.336067530344808 \\ 13.339069021254872 \\ 13.345002652230814 \\ 13.348501107299446 \\ 13.353208250990196 \\ 13.357314388078455 \\ 13.360820356420485 \\ 13.36564509153371 \\ 13.369914250726836 \\ 13.373628814443531 \\ 13.376789366515464 \\ 13.38140069898107 \\ 13.385502461234474 \\ 13.389095528461356 \\ 13.392180442107566 \\ 13.396847742648731 \\ 13.40105107350646 \\ 13.404791372883262\\};

  %%%% p=0.4
\addplot [mark=none,dash pattern=on 10pt off 2pt on 5pt off 6pt,thick] table[x expr=\coordindex,y index=0]{11.873921414691148 \\ 11.346192137179663 \\ 11.016759736780685 \\ 10.78555918623154 \\ 10.615067879894768 \\ 10.486692611320066 \\ 10.389716317992276 \\ 10.31715252836364 \\ 10.263866827172096 \\ 10.226313011769548 \\ 10.201469658795634 \\ 10.187059521140228 \\ 10.181278571778767 \\ 10.182378876134514 \\ 10.189108242541382 \\ 10.200729868481924 \\ 10.215712434202484 \\ 10.23397124038015 \\ 10.254597537421194 \\ 10.277338496080096 \\ 10.301384969459347 \\ 10.326418339596664 \\ 10.352811117067873 \\ 10.379605725128702 \\ 10.40696331218466 \\ 10.434527518764495 \\ 10.461812476345589 \\ 10.489806540741997 \\ 10.517384612691137 \\ 10.545511234948512 \\ 10.572843065161168 \\ 10.600300174358338 \\ 10.626956643667262 \\ 10.653742681968138 \\ 10.680297830473862 \\ 10.706901026467053 \\ 10.73235605125818 \\ 10.75835186606846 \\ 10.783615967859086 \\ 10.808320252378982 \\ 10.83264115666927 \\ 10.856759021187068 \\ 10.88085741709028 \\ 10.904224137461012 \\ 10.926981011018269 \\ 10.950194957205529 \\ 10.973092938205442 \\ 10.994815012925626 \\ 11.016432255658431 \\ 11.038074100612814 \\ 11.0598706589971 \\ 11.079802173437264 \\ 11.101160232576202 \\ 11.120784646122504 \\ 11.140958762509406 \\ 11.160652015682276 \\ 11.179931636867892 \\ 11.200072264353306 \\ 11.218749363405204 \\ 11.237215934973356 \\ 11.255540027318451 \\ 11.273789786828493 \\ 11.29203338698342 \\ 11.309000979220032 \\ 11.326055269826224 \\ 11.343264786133648 \\ 11.360697862278572 \\ 11.376997139602675 \\ 11.39361305770237 \\ 11.410613842770408 \\ 11.426576893037735 \\ 11.443017637676526 \\ 11.458470156575848 \\ 11.474493071220431 \\ 11.489577048449716 \\ 11.50372461860924 \\ 11.520180977199074 \\ 11.534150833072625 \\ 11.548900485190218 \\ 11.562811553516577 \\ 11.577594777063574 \\ 11.591587974867082 \\ 11.606544939831547 \\ 11.618987459966624 \\ 11.634236670389049 \\ 11.64697635646021 \\ 11.66081932419911 \\ 11.673973521230549 \\ 11.686441098122623 \\ 11.700127155028415 \\ 11.713174213542423 \\ 11.725584508187984 \\ 11.737359868242033 \\ 11.750492030641526 \\ 11.763036540669756 \\ 11.774995451866811 \\ 11.786370472095824 \\ 11.799239788626847 \\ 11.811572045859691 \\ 11.823369278791548\\};
  
    %%%% p=0.6
\addplot [mark=none,dotted,thick] table[x expr=\coordindex,y index=0]{11.869074499411866 \\ 11.329989608306311 \\ 10.981666395952804 \\ 10.723816121696794 \\ 10.51904253676326 \\ 10.349157944258396 \\ 10.204042723828742 \\ 10.077441875521526 \\ 9.965069367167954 \\ 9.864185129371876 \\ 9.772616353193332 \\ 9.688848468036202 \\ 9.61175450156113 \\ 9.54027534249631 \\ 9.473700612245352 \\ 9.41164342221909 \\ 9.353209332449424 \\ 9.298400255612508 \\ 9.246694253874455 \\ 9.19795712664407 \\ 9.151720835335302 \\ 9.107802459117256 \\ 9.066401893421885 \\ 9.02696104136108 \\ 8.989552869440047 \\ 8.95395858231973 \\ 8.919890443368443 \\ 8.887866976924164 \\ 8.857249725712634 \\ 8.828541007336776 \\ 8.800989380626811 \\ 8.775073242488645 \\ 8.750270225913706 \\ 8.727065208154272 \\ 8.705244398169503 \\ 8.684942782304907 \\ 8.665489595339809 \\ 8.647793122463 \\ 8.631139644007678 \\ 8.615605809768057 \\ 8.601272568863251 \\ 8.588225354739972 \\ 8.576554219536892 \\ 8.565839265765405 \\ 8.556128911006837 \\ 8.548023518526865 \\ 8.541063034574671 \\ 8.5347189454783 \\ 8.529599454305673 \\ 8.525766329189793 \\ 8.523284284035755 \\ 8.520891614538403 \\ 8.520591346737218 \\ 8.520398069545488 \\ 8.521715278222224 \\ 8.523869271022004 \\ 8.526873032133159 \\ 8.531537198657219 \\ 8.536307340550616 \\ 8.54197349773179 \\ 8.548553613085396 \\ 8.556067103561304 \\ 8.56453490441732 \\ 8.573037470625927 \\ 8.582490341524055 \\ 8.59291871179052 \\ 8.604349430699953 \\ 8.615765878409702 \\ 8.628191105175983 \\ 8.641656629570464 \\ 8.65507126992407 \\ 8.669541003415844 \\ 8.683924193861765 \\ 8.699382853296193 \\ 8.71472441033771 \\ 8.729908224450318 \\ 8.747468234569316 \\ 8.763588685899284 \\ 8.780833166164827 \\ 8.797879900514076 \\ 8.816090799467846 \\ 8.83409189207251 \\ 8.85330215617018 \\ 8.870818031170696 \\ 8.891043587114531 \\ 8.909518163246727 \\ 8.929252061875507 \\ 8.948715111671536 \\ 8.967883024191295 \\ 8.988373323288496 \\ 9.008577360409712 \\ 9.02847420214035 \\ 9.048043702287057 \\ 9.069015032644074 \\ 9.089675438791163 \\ 9.110007802206471 \\ 9.129995691366448 \\ 9.151477957961315 \\ 9.172638355035412 \\ 9.193463123012236\\};  

    %%%% p=1
\addplot [mark=none,solid,thick] table[x expr=\coordindex,y index=0]{11.86907427052243 \\ 11.329987718370322 \\ 10.981659010299452 \\ 10.723795878712751 \\ 10.51899746336892 \\ 10.349070407240182 \\ 10.203888397395684 \\ 10.077188727036472 \\ 9.964676720848551 \\ 9.863602618837213 \\ 9.7717830615387 \\ 9.687691938884456 \\ 9.610189741298711 \\ 9.538204234060796 \\ 9.471010895181111 \\ 9.4082071401056 \\ 9.348884181424244 \\ 9.29302573859479 \\ 9.24009333012553 \\ 9.18993356219852 \\ 9.142061531222888 \\ 9.09627541568821 \\ 9.052750451401412 \\ 9.010913347279288 \\ 8.970813341411448 \\ 8.932212512329214 \\ 8.894807371368625 \\ 8.85907931788153 \\ 8.824383730256061 \\ 8.791181936736017 \\ 8.758726607842037 \\ 8.727454116899434 \\ 8.696844519654313 \\ 8.667335849845957 \\ 8.638702608813434 \\ 8.61104776688786 \\ 8.583729715216482 \\ 8.55756708555451 \\ 8.531884740159228 \\ 8.506738733512378 \\ 8.482187582217422 \\ 8.458292083496076 \\ 8.43511508038063 \\ 8.412277021466677 \\ 8.389814262086407 \\ 8.36822903827628 \\ 8.34711636103872 \\ 8.326032685955079 \\ 8.305484017727691 \\ 8.285513309079741 \\ 8.266164436612492 \\ 8.246433786160365 \\ 8.227909837597483 \\ 8.20903626577507 \\ 8.190908881182514 \\ 8.173011311724972 \\ 8.155362120328135 \\ 8.138564737261632 \\ 8.121480374750751 \\ 8.104703468371108 \\ 8.08825472712243 \\ 8.07215530818825 \\ 8.056426767522984 \\ 8.04044688130311 \\ 8.024862150286411 \\ 8.009695357742922 \\ 7.994969522697877 \\ 7.980023592310645 \\ 7.965545573129992 \\ 7.951559331155252 \\ 7.937374696163295 \\ 7.923710333969238 \\ 7.909856667269403 \\ 7.89655270164304 \\ 7.883069351305753 \\ 7.86940171257709 \\ 7.857093864902493 \\ 7.843848638152472 \\ 7.831220214604293 \\ 7.818430272070656 \\ 7.806289289267033 \\ 7.793999089503996 \\ 7.7823903355874595 \\ 7.769800996003896 \\ 7.758760544157663 \\ 7.746732907753622 \\ 7.735433352499688 \\ 7.724004656676065 \\ 7.71244383427499 \\ 7.701652362642226 \\ 7.690743163541872 \\ 7.679713639966372 \\ 7.668561108015897 \\ 7.658227526161352 \\ 7.647786045440933 \\ 7.637234388789473 \\ 7.62657020629066 \\ 7.616775808698373 \\ 7.60688453121963 \\ 7.596894438144544\\};

 \legend{{\tiny $p=0.2$},{\tiny$p=0.4$},{\tiny$p=0.6$},{\tiny $p=1$}};

 \end{semilogyaxis}
\end{tikzpicture}
\end{center}
\caption{\label{figure:privacy_curve_caps}  Different values for the bound $2 \overline{T}+V(D,M,\overline{T},\eta,p)$ in logarithmic scale, for $\eta \in [n^{-1}, n^{-1/5}]$, $n=1000$, $L=1$, $M=2$, $D=1$.}
\end{figure}

\paragraph{\textbf{Acknowledgments}}

We thank Jason Altschuler for valuable feedback on a first version of this manuscript.

Research partially supported by INRIA Associate Teams project, ANID FONDECYT 1210362 grant, ANID FONDECYT 1241805 grant, ANID Anillo ACT210005 grant, and National Center for Artificial Intelligence CENIA FB210017, Basal ANID.

\bibliographystyle{ims}

%%%%%%%%%%% Bibliography bbl %%%%%%%%%%%%%%%%%%%%%%%

%%%%%%%%%%% Appendix A %%%%%%%%%%%%%%%%%%%%%%%
\newpage
\appendix
\section{Basic Definitions}\label{appendix:definitions_and_results}

\subsection{Information Theory and Probabilistic Divergences}

\begin{definition}[Kullback-Leibler divergence]
    Let $\mu,\nu\in\PP(\RR{d})$ be two probability measures. We define the Kullback-Leibler divergence (abbreviated as KL divergence) as:
    \begin{equation*}
        \mbox{\em KL}(\mu||\nu)=\begin{cases}
            \int_{\RR{d}}\frac{d\mu}{d\nu}(x)\ln\left(\frac{d\mu}{d\nu}(x)\right)\nu(dx)\text{ if }\mu\ll\nu\\
            +\infty\text{ otherwise}
        \end{cases}
    \end{equation*}
\end{definition}

It is a well-known fact \citep[Theorem 5]{van2014renyi} that one can extend, by taking limits, the R\'enyi divergence to the case $\alpha=1$ when $R_{\beta}<\infty$ for some $\beta>1$. This extreme case results in the KL divergence,
i.e.~$R_1(\mu||\nu)=\mbox{KL}(\mu||\nu).$

\begin{definition}[Total Variation distance]
    Let $\mu,\nu\in\PP(\RR{d})$ be two probability measures. We define the total variation distance between $\mu$ and $\nu$ as:
    \begin{equation*}
        \normtv{\mu-\nu}=\sup_{A\in\mathcal{B}(\RR{d})}\left|\mu(A)-\nu(A)\right|.
    \end{equation*}
\end{definition}

Note that if a sequence of probability measures $(\nu_n)_{n\in\mathbb{N}}\subset\PP(\XX)$ converges in total variation to a measure $\nu\in\PP(\XX)$, then

\begin{equation*}
    \lim_{n\to\infty}\int_{\XX}g(x)\nu_n(dx)=\int_{\XX}g(x)\nu(dx)
\end{equation*}

for all measurable and bounded function $g$ with support on $\XX$. This is because a measurable and bounded function can be uniformly approximated by a simple functions.

A useful inequality that compares total variation with KL divergence is Pinsker's inequality \citep[Lemma 2.5]{Tsybakov_2009}:

\begin{proposition}[Pinsker's inequality]\label{prop:pinskers_inequality} Let $\mu,\nu\in\PP(\RR{d})$ be two probability measures. Then
    \begin{equation*}
        \normtv{\mu-\nu}\leq \sqrt{\frac{1}{2}\mbox{\em KL}(\mu\|\nu)}.
    \end{equation*}
\end{proposition}

Another useful inequality for comparing KL divergence and total variation is Bretagnolle-Huber's inequality \citep[Lemma 3]{canonne2022short}. Although this inequality is worse than Pinsker's when the KL divergence moves between $0$ and $2$, it has the advantage of being nonvacous when it exceeds this limit.

\begin{proposition}[Bretagnolle-Huber inequality]\label{lemma:Bretagnole-Huber_inequality} Let $\mu,\nu\in\PP(\RR{d})$ be two probability measures. Then
    \begin{equation*}
        \normtv{\mu-\nu}\leq\sqrt{1-\exp\left (-\mbox{\em KL}(\mu||\nu) \right )}.
    \end{equation*}
\end{proposition}
The following is the well-known data-processing inequality \citep[Theorem 9]{van2014renyi} for R\'enyi divergence:
\begin{proposition}\label{prop:data_processing_ineq}
    Let $P:\RR{d}\to\PP(\RR{d})$ be a measurable map and let $J:\PP(\XX)\to\PP(\XX)$ be the transition operator associated to $P$ (see Definition \ref{def:transition_operator}). Then, for all  $\mu,\nu\in\PP(\RR{d})$, $R_{\alpha}\left(J\mu||J\nu\right)\leq R_{\alpha}(\mu||\nu)$.
\end{proposition}

Let $X_1,\ldots,X_n$ be a sequence of (possibly random) vectors. We abbreviate $X_1,\ldots,X_k$ by $X_{1:k}$. The following proposition corresponds to \citet[Lemma 2.9] {altschuler2022privacy}.    
\begin{proposition}[Strong composition]\label{prop:strong_composition}
    Set $\alpha\geq 1$ and  let $X_{1:k}$ and  $Y_{1:k}$ be two sequences of random variables. Then
    \begin{equation*}
        R_{\alpha}(\proba_{X_{1:k}}\|\proba_{Y_{1:k}})\leq\sum \nolimits_{i=1}^{k}\sup_{x_{1:i-1}}R_{\alpha}\left(\proba_{X_i|X_{1:i-1}=x_{1:i-1}}\|\proba_{Y_i|Y_{1:i-1}=x_{1:i-1}}\right).
    \end{equation*}
\end{proposition}

\subsection{Mixing Times}

We start by introducing the terminology and basic results regarding homogeneous Markov chains (HMC). For more information, we refer the reader to \citet{hairer2006ergodic}.

\begin{definition}[Homogeneous Markov Chain, Transition Probabilities] We say\\
that a Markov Chain taking values on a set $\XX\subseteq\RR{d}$, $(X_t)_{t\in\NNo}$, is (time) homogeneous if there exists a measurable map $P:\XX\to\PP(\XX)$ such that:
    \begin{equation*}
        \prob{X_t\in A|X_{t-1}=x}=P(x,A)
    \end{equation*}
    for every $A\in{\BB}(\XX)$, almost every $x\in\XX$, and every $t\geq 1$. The map $P$ from above is called the transition probabilities of the chain.
\end{definition}

    We will usually call transition probabilities to all measurable maps $P:\XX\to\PP(\XX)$, even if no HMC is specified. This is justified since for every such map there exists an HMC that has it as transition probabilities (see, for example, \citealt[Proposition 2.38]{hairer2006ergodic}).

\begin{proposition}[\citealt{hairer2006ergodic}, Theorem 2.29]
    Let $(X_t)_{t\in\mathbb{N}_{0}}$ be an HMC taking values on $\XX$ and with transition probabilities $P$. Then, for all $t\geq 1$,
    \begin{equation*}
        \prob{X_t\in A|X_0=x}=P^t(x,A),
    \end{equation*}
    where $P^t$ is defined recursively by:
    \begin{equation*}
        P^t(x,A)=\int_{\XX}P(z,A)P^{t-1}(x,dz).
    \end{equation*}
\end{proposition}

An easy consequence of the previous proposition is that
\begin{equation*}
    P^{t+s}(x,A)=\int_{\XX}P^{t}(z,A)P^{s}(x,dz)
\end{equation*}
for all $t,s\geq 1$.

\begin{definition}[Transition Operator]\label{def:transition_operator}
    Given transition probabilities $P:\XX\to\PP(\XX)$, we define the transition operator $J:{\PP}(\XX)\mapsto {\PP}(\XX)$ by:
    \begin{equation*}
        (J\mu)(A)=\int_{\XX}P(z,A)\mu(dz),
    \end{equation*}
    for every $A\in{\BB}(\XX)$.
\end{definition}

\begin{remark}
    If $(X_t)_{t\in\NNo}$ is an HMC with transition probabilities $P$ and transition operator $J$ such that $X_0\sim\mu_0$, then the distribution of $X_t$, for $t\geq 1$, is the one that for all $A\in{\BB}(\XX)$:
    \begin{equation*}
        J^t\mu_0(A)=\int_{\XX}P^t(z,A)\mu_0(dz),
    \end{equation*}
    as one can check.
\end{remark}

\begin{definition}[Invariant measure]
    Given a transition operator $J$, we say that the measure $\pi$ is an invariant (or stationary) measure of $J$ if
    \begin{equation*}
        J\pi=\pi.
    \end{equation*}
\end{definition}

When we we talk about the HMC (instead of its transition operator), we usually call $\pi$ the \textbf{stationary distribution} of $(X_t)_{t\in\mathbb{N}_0}$, instead of the invariant measure (of its transition operator). This is justified by the following Proposition:

\begin{proposition}
    Let $(X_t)_{t\in\mathbb{N}_0}$ be an HMC and let $P$ and $J$ be its transition probabilities and its transition operator, respectively. If $\pi$ is the invariant measure of $J$ and $X_0\sim\pi$, then:
\begin{equation*}
    X_t\sim\pi\quad (\forall t\in\mathbb{N}).
\end{equation*}
\end{proposition}

\begin{definition}[Mixing time]\label{def:mixing_time}
    Let $(X_t)_{t\in\mathbb{N}_0}$ be an HMC with transition probabilities $P$ and stationary distribution $\pi$. We define the mixing time in total variation up to error $\eps>0$ of the chain as:
    \begin{equation*}
        T_{mix,TV}(\eps):=\min\{t\in\mathbb{N}:d(t)\leq \eps\},
    \end{equation*}
    where
    \begin{equation*}
        d(t):=\sup_{x\in\XX}\normtv{P^t(x,\cdot)-\pi}.
    \end{equation*}
\end{definition}

\begin{proposition}[\citealt{levin2017markov}, Section 4.5]\label{prop:appendix_boosting}
    Let $(X_t)_{t\in\mathbb{N}_0}$ be an HMC supported on a compact set $\XX$ and with stationary distribution $\pi$. Let $$\Bar{d}(t):=\sup_{x,y\in\XX}\normtv{P^t(x,\cdot)-P^t(y,\cdot)}.$$
    If $T^{\ast}$ is such that:
    \begin{equation*}
        \Bar{d}(T^{\ast})\leq\frac{1}{2},
    \end{equation*}
    then:
    \begin{equation*}
        T_{mix,TV}(\eps)\leq T^{\ast}\cdot\left\lceil \log_2(1/\eps)\right\rceil.
    \end{equation*}
\end{proposition}

We highlight the distance $\Bar{d}$ also holds when the points $x,y\in\XX$ are replaced by probability measures $\mu,\nu\in\PP(\XX)$, i.e.
\begin{equation*}
    \Bar{d}(t)=\sup_{\mu,\nu\in\PP(\XX)}\normtv{J^t\mu-J^t\nu}.
\end{equation*}

\begin{definition}[Dual operator]
    We define the dual operator of $J$, denoted $J_{\ast}$ as:
    \begin{equation*}
        (J_{\ast} f)(x)=\expected{}{f(X_1)|X_0=x}=\int_{\XX}f(z)P(x,dz).
    \end{equation*}
\end{definition}

It should be noted that, for all bounded and measurable function $g$ and all probability measure $\mu\in\PP(\XX)$, the dual operator satisfies:
\begin{equation*}
    \int_{\XX} (J_{\ast} g)(x)\mu(dx)=\int_{\XX}g(x)(J\mu)(dx),
\end{equation*}
and that it sends bounded and measurable functions into bounded and measurable functions.

\newpage
\subsection{Differential Privacy}

\subsubsection{Privacy Amplification by Sampling}

\begin{definition}[R\'enyi Divergence of the Sampled Gaussian Mechanism] \text{ }
    
    \noindent Let $\alpha\geq 1$ be a R\'enyi parameter, $q\in (0,1)$ be a mixture parameter and $\sigma>0$ be a noise level. Define
    \begin{equation*}
        S_{\alpha}(q,\sigma):=R_{\alpha}\left(\NN(0,\sigma^2)||(1-q)\NN(0,\sigma^2)+q\NN(1,\sigma^2)\right).
    \end{equation*}
\end{definition}

\begin{lemma}[\citealt{altschuler2022privacy}, Lemma 2.11]\label{lemma:bound_mixture}
    Let $\alpha\geq 1$ be a R\'enyi parameter, $q\in (0,1)$ be a mixture parameter, $\sigma>0$ be a noise level, $d\in\mathbb{N}$ be the dimension and $r>0$ be a radius. Then:
    \begin{equation*}
        \sup_{\mu\in\PP(B(0,r))}R_{\alpha}\left(\NN(0,\sigma^2 I_{d\times d})||(1-q)\NN(0,\sigma^2 I_{d\times d})+q\left(\NN(0,\sigma^2 I_{d\times d})\ast\mu\right)\right)=S_{\alpha}(q,\sigma/r),
    \end{equation*}
    where $B(0,r)$ denotes the Euclidean $d$-dimensional closed ball centered at the origin and with radius $r$.
\end{lemma}

\begin{lemma}[\citealt{mironov2019r}, Theorem 11]\label{lemma:bound_of_S_alpha}
    Let $\alpha\geq 1$ be a R\'enyi parameter, $q\in (0,1/5)$ be a mixture parameter and $\sigma\geq 4$ be a noise level. If $\alpha\leq \alpha^{\ast}(q,\sigma)$, then:
    \begin{equation*}
        S_{\alpha}(q,\sigma)\leq 2\alpha q^2/\sigma^2,
    \end{equation*}
    where $\alpha^{\ast}(q,\sigma)$ is the largest $\alpha$ satisfying:
    \begin{align*}
        \alpha &\leq \frac{M\sigma^2}{2}-\log(\sigma^2)\\
        \text{and}\quad \alpha &\leq \frac{M^2\sigma^2/2-\log(5\sigma^2)}{M+\log(q\alpha)+1/(2\sigma^2)},
    \end{align*}
    with:
    \begin{equation*}
        M=\log\left(1+\frac{1}{q(\alpha-1)}\right).
    \end{equation*}
\end{lemma}

%%%%%%%%%%% Appendix B %%%%%%%%%%%%%%%%%%%%%%%

\section{Nonconvexity of \texorpdfstring{$E$}{E} for the convex weakly smooth case} \label{app:convexity_E}

Recall that:
\begin{equation*}
    E (\uu):=\sum_{t=1}^{T}\frac{(\varphi_{t-1}(u_{t-1})-u_{t})^2}{\sigma_{t-1}^2}.
\end{equation*}

Let's call $g_{t-1}$ to each of addends of $E$, except for the fist and the last; i.e. for each $t=2,\ldots,T-1$, let $g_{t-1}(u_{t-1},u_t):=(\varphi_{t-1}(u_{t-1})-u_{t})^2$.

\begin{proposition}
    If $\varphi_t(\delta)=\sqrt{\delta^2+h_t}$, then the Hessian of $g_{t-1}$ is:
    \begin{equation*}
        \nabla^2 g_{t-1}(u_{t-1},u_t)=\begin{pmatrix}
            2-\frac{2h_t u_t}{(u_{t-1}^2+h_t)^{3/2}} & -\frac{2u_{t-1}}{\sqrt{u_{t-1}^2+h_t}}\\
            -\frac{2u_{t-1}}{\sqrt{u_{t-1}^2+h_t}} & 2
        \end{pmatrix}
    \end{equation*}
    The determinant and trace of this Hessian are:
    \begin{align*}
        \det\nabla^2 g_{t-1} &= \frac{4h_t\left(\sqrt{u_{t-1}^2+h_t}-u_t\right)}{(u_{t-1}^2+h_t)^{3/2}}\\
        \text{Tr}\nabla^2 g_{t-1} &= 4-\frac{2h_t u_t}{(u_{t-1}^2+h_t)^{3/2}}.
    \end{align*}
    Moreover, $\nabla^2 E(\uu)$ is positive semidefinite when $\uu\in\RRR$.
\end{proposition}

Even though the Hessian of $E$ is positive semidefinite over ${\RRR}$, it is easy to see that in the case $\varphi_t(\delta)=\sqrt{\delta^2+h_t}$, the feasible set ${\RRR}$ is nonconvex. This prevents the shifts optimization problem from being convex.

Moreover, as can be seen in Figure \ref{fig:E_is_not_convex}, $E$ is not convex over $\RR{T-1}$. So even if we prove that $u^{\ast}\in\RRR$, through second order conditions we can only guarantee that it is a local minimum.

\begin{figure}
\begin{center}
    \begin{tikzpicture}
\begin{axis}[
    xlabel={$u_1$},
    axis line style = thick,
    axis lines = middle,
    axis x line=bottom,
    width=10cm,
    height=7cm,
    domain=0:3,
    xmin= 0,
    xmax=3.1,
    ymin= 10,
    ymax=130,
    samples=100,
    legend pos=north east
]
\addplot[
    thick,
    black,
] 
{
    ((sqrt(1^2 + 4) - x)^2 / 1^2) + 
    ((sqrt(x^2 + 4) - 3)^2 / 0.1^2) + 
    ((sqrt(3^2 + 4))^2 / 1^2)
};
\addlegendentry{$E(u_1,3)$}
\end{axis}
\end{tikzpicture}
\end{center}
    \caption{$E(u_1,3)$ with $D=1$, $\eta\equiv 1$, $L=1$, $\sigma_0=1$, $\sigma_1=0.1$, $\sigma_2=1$.}
    \label{fig:E_is_not_convex}
\end{figure}

%%%%%%%%%%% Appendix C %%%%%%%%%%%%%%%%%%%%%%%

\section{Existence of Stationary Distributions for Nondifferentiable Potentials}\label{appendix:existence_of_stationary_distributions}

When the potential $f$ is in ${\CC}^1$, the existence of a stationary distribution follows by standard results, which are based on the Feller condition (see, for example, \citealt[Theorem 4.22 and Corollary 4.18]{hairer2006ergodic}). Since the potentials with H\"older continuous gradients fall in this case, we will focus only in the Lipschitz case.

If one only asks to the potential $f$ to be Lipschitz, it is no longer necessary that it is differentiable. Of course, since our potentials are always convex, $f$ will be subdifferentiable. In order to keep notation simple, we will use $\nabla f(x)$ to denote a subgradient of $f$ in $x$. We will assume that we have access to an oracle that selects such subgradient and that is consistent with future choices; with this we mean that if the oracle have access to the same point in to different iterations of the algorithm, it will give the same subgradient.

\begin{lemma}\label{lemma:cauchy_sequence_probability_map}
    Let $\XX\subseteq\RR{d}$ be a convex, compact set with diameter $D>0$ and suppose $f:\XX\subseteq\RR{d}\to\RR{}$ is a subdifferentiable function. Let $P$ be the transition probabilities of the HMC defined by:
    \begin{equation*}
        X_{t+1}=\Pi_{\XX}\left[X_t-\eta\nabla f(X_t)+\sqrt{2\eta}\xi_{t}\right],
    \end{equation*}
    where $\eta>0$ and $(\xi_{t})_{t\in\NNo}\overset{i.i.d.}{\sim}\NN\left(0,I_{d\times d}\right)$, then, for every $x\in \XX$, the sequence $(P^{t}(x,\cdot))_{t\in\mathbb{N}}$ is a Cauchy sequence with respect to the total variation norm.
\end{lemma}

\begin{proof}
    Denote by $J$ the transition operator associated with $P$. By Lemma \ref{lemma:nombre_pendiente} applied to the projection (which is nonexpansive) and Lemma \ref{lemma:shift_reduction} with $\alpha=1$ and $a=D$, we have that for any $\mu,\nu\in\mathcal{P}(\XX)$:
    \begin{equation*}
        \mbox{KL}(J\mu||J\nu)\leq \frac{D^2}{4\eta}.
    \end{equation*}

    Then, by Bretagnolle-Huber inequality (Proposition \ref{lemma:Bretagnole-Huber_inequality}) 
    \begin{equation*}
        \Bar{d}(1)=\sup_{\mu,\nu\in\mathbb{P}(\XX)}\normtv{J\mu-J\nu}\leq \sqrt{1-\expo{-\frac{D^2}{4\eta}}}=:\kappa<1.
    \end{equation*}

    Let $\eps>0$ and take $l\geq \left\lceil\frac{\ln(1/\eps)}{\ln(1/\kappa)}\right\rceil$. Then, by the submultiplicativity of $\bar{d}$ \citep[Lemma 4.11]{levin2017markov}, we have that
    \begin{equation*}
        \sup_{\mu,\nu\in\PP(\XX)}\normtv{J^{ l}\mu-J^{ l}\nu}=\bar{d}(l)\leq
        \bar{d}(1)^{l}\leq \kappa^{l}<\eps.
    \end{equation*}
    Fix $x\in \XX$. Taking $n\geq l$ and $m>j\geq 1$, we have that:
    \begin{align*}
        \normtv{P^{n+m}(x,\cdot)-P^{n+j}(x,\cdot)}
        &=
        \normtv{J^{n- l}\left(J^{ l}P^{m}(x,\cdot)-J^{ l}P^{j}(x,\cdot)\right)}\\
        &\leq
        \normtv{J^{l}P^{m}(x,\cdot)-J^{l}P^{j}(x,\cdot)}\\
        &<\eps
    \end{align*}
    where the first inequality follows by the data processing inequality and the second by the way we chose $l$.
\end{proof}

In order to save space, in the proof of the following theorem we will use sometimes the bracket notation to denote integrals with respect a measure. That is, if $g:\XX\to\RR{}$ is an integrable function and $\nu\in\PP(\XX)$:
\begin{equation*}
    \left<g,\nu\right>=\int_{\XX}g(x)\nu(dx).
\end{equation*}

The following theorem is based on the proof of \citealt[Theorem 4.17]{hairer2006ergodic}.

\begin{theorem}
    Let $\XX\subseteq\RR{d}$ be a convex, compact set with diameter $D>0$ and suppose $f:\XX\to\RR{}$ is a subdifferentiable potential. Then the HMC defined by:
    \begin{align*}
        X_0 &\sim \mu_0\in\PP(\XX)\\
        X_{t+1}&=\Pi_{\XX}\left[X_t-\eta\nabla f(X_t)+\sqrt{2\eta}\xi_{t}\right],
    \end{align*}
    where $\xi_{t}\sim\NN\left(0,I_{d\times d}\right)$, has a stationary distribution $\pi_{\eta}$.
\end{theorem}

\begin{proof}
    Take $x\in \XX$. By Lemma \ref{lemma:cauchy_sequence_probability_map} and the completeness of the total variation norm, there exists a measure $\pi_{\eta}\in \PP(\XX)$ such that $\lim_{n\to\infty}\normtv{P^n(x,\cdot)-\pi_\eta}=0$.
    
    By C\`esaro convergence, the sequence of measures defined by $\nu_n = \frac{1}{n}\sum_{k=1}^{n}P^k(x,\cdot)$ also converges to $\pi_{\eta}$ in total variation. Notice that the sequence $(\nu_n)_{n\in\mathbb{N}}$ satisfies the identity
    \begin{equation}\label{eqn:defining_relationship_of_nun}
        J\nu_n -\nu_n= \frac{1}{n}\left[P^{n+1}(x,\cdot)-P(x,\cdot)\right].
    \end{equation}

    We will now prove that $J\pi_{\eta}=\pi_{\eta}$. In order to do this, we will prove that for every continuous and bounded function $g$ it holds that
    \begin{equation}\label{ineq:pi_is_a_fixed_point}
        \left|\int_{\XX}g(x)(J\pi_{\eta})(dx)-\int_{\XX}g(x)\pi_{\eta}(dx)\right|<\eps
    \end{equation}
    for every $\eps>0$.

    Let's fix a bounded and continuous function $g$. We can decompose the left side of \eqref{ineq:pi_is_a_fixed_point} into three easier to bound parts via triangle inequality:
    \begin{equation}\label{ineq:analyzing_jpi_pi}
        \left|\langle g,J\pi_{\eta}\rangle-\langle g,\pi_{\eta}\rangle\right|\leq \left|\langle g,J\pi_{\eta}\rangle -\langle g,J\nu_n\rangle\right|+\left|\langle g,J\nu_n\rangle-\langle g, \nu_n\rangle\right|+\left|\langle g,\nu_n\rangle-\langle g,\pi_{\eta}\rangle\right|,
    \end{equation}
    Let us define $n_1,n_2,n_3\in\mathbb{N}$ as:
    \begin{enumerate}
        \item $n_1$ is such that $\forall m\geq n_1$:
        \begin{equation*}
            \left|\left<J_{\ast} g,\pi_\eta\right>-\left<J_{\ast}g,\nu_m\right>\right|<\frac{\eps}{3}.
        \end{equation*}
        This quantity exists because $J_{\ast}g$ is a bounded and measurable function and $\nu_m$ converges in total variation to $\pi_\eta$, so the integral converge.

        \item $n_2$ is such that $\forall m\geq n_2$:
        \begin{equation*}
            \frac{2\max_{x\in \XX}\left|g(x)\right|}{m}<\frac{\eps}{3}.
        \end{equation*}
        The existence of such quantity is obvious. The relevance of this inequality comes from the following:
        \begin{align*}
            \left|\langle g,J\nu_m\rangle-\langle g,\nu_m\rangle\right| &= \left|\langle g,J\nu_m-\nu_m\rangle\right|\\
            &=\left| \left< g, \frac{1}{m}(P^{m+1}(x,\cdot)-P(x,\cdot))\right>\right|\\
            &\leq
            \frac{2\max_{x\in \XX}|g(x)|}{m}\\
            &<\frac{\eps}{3}
        \end{align*}
        for all $m\geq n_2$, where the second equality comes from \eqref{eqn:defining_relationship_of_nun}.

        \item $n_3$ is such that for all $m\geq n_3$
        \begin{equation*}
            \left|\langle g,\nu_m\rangle -\langle g,\pi_{\eta}\rangle \right|<\frac{\eps}{3}.
        \end{equation*}
        This quantity exists since $\nu_m$ converges in total variation to $\pi_{\eta}$, which implies that the integral converge.
    \end{enumerate}
    Taking $n\geq\max\{n_1,n_2,n_3\}$ in \eqref{ineq:analyzing_jpi_pi}, we conclude that $
         \left|\langle g,J\pi_{\eta}\rangle-\langle g,\pi_{\eta}\rangle\right| < \eps$ and therefore $J\pi_{\eta} =\pi_{\eta}$.
\end{proof}

\end{document}